%% file: main.tex
\documentclass[twoside]{article}

\usepackage[preprint]{aistats2027}
\usepackage{natbib}
\usepackage{booktabs}
\usepackage{amsmath}
\usepackage{amssymb}

\input{notation}

\begin{document}

\twocolumn[

\aistatstitle{PAC-Causal Learnability: From Identification to Finite-Sample Guarantees}

\aistatsauthor{ Mian Wei \And Somesh Jha \And  David Page}

\aistatsaddress{ Duke University \And  University of Wisconsin–Madison \And Duke University } ]

\input{sections/00_abstract}
\input{sections/01_intro}
\input{sections/02_capac}

\input{sections/04_sccs}
\input{sections/05_IV}
\input{sections/06_ps}
\input{sections/03_relatedwork}
\input{sections/07_discussion}

\bibliographystyle{apalike}
\bibliography{refs}

\clearpage
\appendix
\thispagestyle{empty}

\onecolumn
\aistatstitle{PAC-Causal Learnability: From Identification to Finite-Sample Guarantees \\
Supplementary Materials}

\input{sections/09_appendix}

\end{document}

%% file: notation.tex
\newcommand{\A}{\mathcal A}

\usepackage{algorithm}
\usepackage{algorithmic}
\usepackage{amsmath}
\usepackage{amssymb}
\usepackage{amsthm}
\usepackage{multirow}

\newtheorem{theorem}{Theorem}
\newtheorem{lemma}{Lemma}
\newtheorem{corollary}{Corollary}

\theoremstyle{definition}
\newtheorem{definition}{Definition}
\newtheorem{assumption}{Assumption}
\newtheorem{proposition}{Proposition}

%% file: sections/00_abstract.tex
\begin{abstract}

Causal identification asks whether a causal target can be recovered from the observational distribution in an idealized infinite-data setting. In practice, only finite samples are available, and identifiability alone does not determine how much data are needed for reliable estimation. We introduce Probably Approximately Correct Causal (PACC) learnability, which formalizes uniform accuracy--confidence guarantees over a causal model class. From a PAC-learning perspective, we ask how many samples are sufficient to estimate a causal target to accuracy \(\epsilon\) with probability at least \(1-\delta\).
We show that PACC sample complexity is shaped by three recurring factors: causal identification strength, target complexity, and nuisance complexity, and derive a general sufficient bound linking them to finite-sample causal learning. We illustrate the framework through self-controlled case series, instrumental variables, and propensity-score-based estimation, showing how distinct causal methods can be analyzed within a common framework. 
\end{abstract}

%% file: sections/01_intro.tex
\section{Introduction}

Causal inference aims to estimate the effects of treatments or interventions and is central to statistics, epidemiology, economics, and many other fields \citep{pearl2009causality,imbens2015causal,hernan2020causal,angrist2009mostly}. In observational studies, causal estimation requires additional assumptions and methods to address confounding and other sources of bias. A broad range of approaches have been developed, including regression adjustment, propensity-score methods, instrumental variables (IV), self-controlled designs, and many others \citep{rosenbaum1983central,angrist1996identification,farrington1995relative,imbens2015causal,hernan2020causal}. For many of these methods, classical theoretical guarantees are formulated in terms of identifiability, consistency, and asymptotic distributions as the sample size tends to infinity. Such results focus on whether a causal target is identifiable and on the limiting behavior of an estimator, but they do not directly answer a finite-sample question: \textbf{how many observational samples are needed to estimate a causal target at a prescribed accuracy and confidence level?}

This question has received partial answers. Classical sample-size and power analyses provide practical guidance but are often based on large-sample approximations \citep{shooksa2022power,liu2026sample}, while recent nonasymptotic or PAC-style guarantees have largely focused on specific causal settings, targets, or estimation procedures \citep{bhattacharyya2022efficient,choo2025probably,yin2025sampling}. See the related work for further discussion. What remains less developed is a common notion of finite-sample causal learnability and an understanding of whether different causal methods share common drivers of sample complexity.

Inspired by Probably Approximately Correct (PAC) learning \citep{valiant1984theory}, we introduce Probably Approximately Correct Causal (PACC) learnability, which formalizes uniform finite-sample accuracy--confidence guarantees over a causal model class. Within this framework, we identify three recurring factors shaping causal sample complexity: causal identification strength, which measures how well observational distributions separate meaningfully different causal targets; target complexity, which captures the richness of the estimand; and nuisance complexity, which reflects the burden of estimating auxiliary components required for causal estimation. Under general conditions, we derive a sufficient PACC bound linking these factors to the sample size required for a prescribed $(\epsilon,\delta)$ guarantee.

We illustrate the framework through self-controlled case series (SCCS), many-instrument IV, and propensity-score-based CATE estimation. These examples progressively highlight different components of the framework: SCCS is driven primarily by identification strength, many-IV introduces a nuisance-estimation burden, and propensity-score-based CATE estimation brings richer target complexity into play. Despite their different assumptions and estimators, all three can be analyzed through the same PACC perspective. In this sense, PACC provides an alternative route to finite-sample guarantees that complements classical asymptotic analysis.

Overall, our contributions are threefold. First, we extend population-level causal identifiability to a unified finite-sample notion of causal learnability. Second, we identify three common factors shaping causal sample complexity and derive a general sufficient PACC bound. Third, we show how these factors specialize across distinct observational causal methods.

%% file: sections/02_capac.tex
\section{PAC-Causal Learnability}
\label{sec:pacc-framework}

Let $\mathcal M$ be a class of causal data-generating models. Each $m\in\mathcal M$ specifies an observational distribution $P_m$ on an observation space $\mathcal O$ and a causal target $\phi(m)$ in a target space $\Theta$. A learner $\mathcal A$ observes an i.i.d.\ sample $S_n=(O_1,\ldots,O_n)\sim P_m^{\otimes n}$ and returns an estimate $\widehat\phi_n=\mathcal A(S_n)\in\Theta$. Throughout, $a\lesssim b$ means $a\le Cb$ for a constant $C>0$ independent of $n,\epsilon,\delta$ and of the varying problem parameters (e.g., $k$, $\lambda$, $\alpha$) unless a subscript indicates otherwise; $a\gtrsim b$ is defined analogously, and $a\asymp b$ means both hold.

We equip $\Theta$ with a metric $\rho$ to specify estimation error. An estimate is $\epsilon$-accurate when $\rho(\widehat\phi_n,\phi(m))\le\epsilon$. For a scalar effect, $\rho$ can be absolute error; for a vector-valued target, it can be the Euclidean norm; and for a function-valued target, it can be an $L_2$ norm.
Unlike standard supervised learning, causal learning is inherently a missing-data problem: for each unit, outcomes under alternative interventions are not jointly observed. Causal assumptions encoded in $\mathcal M$ are therefore needed to link the observed-data distribution $P_m$ to the causal target $\phi(m)$, motivating a finite-sample notion of reliable causal learning.

\subsection{PAC-Causal Learnability}

We provide the following definition of PACC learnability.

\begin{definition}[PAC-Causal learnability]
For $\epsilon>0$, $\delta\in(0,1)$ and $n\ge1$, we say that an $(\epsilon,\delta)$-PACC guarantee holds at sample size $n$ if there is a learner $\A_n$ with
\[
\sup_{m\in\mathcal M}P_m^{\otimes n}\{\rho(\A_n(S_n),\phi(m))>\epsilon\}\le\delta .
\]
Let $n^*_{\mathrm{PACC}}(\epsilon,\delta)$ be the smallest such $n$. Since a learner may ignore extra observations, the guarantee then holds at every $n\ge n^*_{\mathrm{PACC}}(\epsilon,\delta)$. The problem $(\mathcal M,\phi,\rho)$ is PAC-Causal (PACC) learnable if $n^*_{\mathrm{PACC}}(\epsilon,\delta)<\infty$ for every $\epsilon>0$ and $\delta\in(0,1)$.
\end{definition}

PACC learnability implies identifiability and the existence of a pointwise consistent estimator. Identifiability requires only $P_m=P_{m'}\Rightarrow\phi(m)=\phi(m')$, a population-level statement about whether recovery is possible in principle, but not how many observations are needed for a prescribed accuracy. Pointwise consistency requires $P_m^{\otimes n}\{\rho(\mathcal A(S_n),\phi(m))>\epsilon\}\to0$ for each fixed $m\in\mathcal M$. Thus, for every fixed model and every $(\epsilon,\delta)$, there exists some model-specific sample size $N_m(\epsilon,\delta)$, but $N_m(\epsilon,\delta)$ may become arbitrarily large as $m$ varies over $\mathcal M$. PACC learnability instead asks whether these model-specific requirements admit a single finite upper bound $N(\epsilon,\delta)$ over the entire model class.

A weak-IV example illustrates the distinction. Consider $X=\pi Z+U_X$ and $Y=\beta X+U_Y$, where $Z\sim N(0,1)$ is independent of $(U_X,U_Y)$, $U_X$ and $U_Y$ may be correlated. Let the target be $\phi(m)=\beta$, with $0<|\pi|\le1$ and $|\beta|\le B<\infty$. For every fixed model with $\pi\neq0$, $\beta=\operatorname{Cov}(Z,Y)/\operatorname{Cov}(Z,X)$, so the effect is identified and the Wald estimator is pointwise consistent.

PACC learnability nevertheless fails when the instrument may be arbitrarily weak. Fix an effect separation $\Delta>0$. For each $\pi$, let $m_\pi^+$ and $m_\pi^-$ denote two admissible models whose causal effects differ by $\Delta$, and let $P_\pi^+$ and $P_\pi^-$ be their corresponding observational laws. Let $h^2(P,Q)=1-\int\sqrt{dP\,dQ}$ denote squared Hellinger distance. By appropriately adjusting the confounding, these models can be chosen so that
\[
h^2(P_\pi^+,P_\pi^-)=O(\pi^2\Delta^2)\to0
\qquad\text{as }|\pi|\to0,
\]
while $|\phi(m_\pi^+)-\phi(m_\pi^-)|=\Delta$. Thus, although every fixed model with $\pi\neq0$ is identified and pointwise consistently estimable, for any fixed sample size $n$ the model class contains causally distinct models whose observational distributions are arbitrarily difficult to distinguish. A standard two-point argument therefore rules out a uniform $(\epsilon,\delta)$ guarantee for $\epsilon<\Delta/2$ and $\delta<1/2$. The full construction is given in the appendix~A.

\subsection{From Identifiability to Finite-Sample Learnability}
\label{sec:finite-sample}

Identifiability is qualitative, whereas PACC learning requires quantitative finite-sample control. With finitely many observations, causal targets that differ by more than $\epsilon$ must be uniformly separated in the induced observational distributions, and the learner must also account for the number of competing target configurations and errors from nuisance estimation. 

\paragraph{Causal identification strength.}
As illustrated by weak instrumental variables, models with very different causal effects can induce arbitrarily close observational distributions. We therefore quantify the smallest observational separation associated with an $\epsilon$-scale difference in the causal target. 

\begin{definition}[Causal identification strength]
\label{def:id-strength}
For $\epsilon>0$, define
\[
S_{\mathcal M}(\epsilon)
=
\inf_{\substack{m,m'\in\mathcal M\\
\rho(\phi(m),\phi(m'))>\epsilon}}
h^2(P_m,P_{m'}),
\]
with the convention $\inf\varnothing=1$.
\end{definition}

We use squared Hellinger distance because it is symmetric and tensorizes under i.i.d.\ sampling: $1-h^2(P^{\otimes n},Q^{\otimes n})=\{1-h^2(P,Q)\}^n$. Smaller $S_{\mathcal M}(\epsilon)$ therefore indicates weaker identification at accuracy level $\epsilon$, and this separation yields an estimator-independent lower bound.

\begin{proposition}[Identification barrier]
\label{prop:id-barrier}
For \(0<\delta\le 1/8\), if \(0<S_{\mathcal M}(2\epsilon)<1/2\), then
\[
n_{\mathrm{PACC}}^*(\epsilon,\delta)
\gtrsim
\frac{\log(1/\delta)}
     {S_{\mathcal M}(2\epsilon)}.
\]
If \(S_{\mathcal M}(2\epsilon)=0\), then no finite sample size provides an \((\epsilon,\delta)\)-PACC guarantee.
\end{proposition}

\paragraph{Target complexity.}
Identification strength describes the difficulty of distinguishing \emph{one pair} of $\epsilon$-separated causal targets. A learner, however, generally faces an entire family of possible targets. Even when every pair is sufficiently separated, the problem becomes harder when many distinct target configurations must be distinguished simultaneously.

We write the observational law $P_m$ more explicitly as $P_{\theta,\eta}$, where $\theta=\phi(m)\in\Theta$ is the causal target and $\eta\in\mathcal H$ denotes nuisance structure. For fixed $\eta$, let $\mathcal P_\eta=\{P_{\theta,\eta}:\theta\in\Theta\}$. Because $\mathcal P_\eta$ may be infinite, we summarize its effective size using a covering number. For $r>0$ and a class of distributions $\mathcal P$, define
\[
\begin{aligned}
N(r,\mathcal P,h)
=
\min\bigl\{K:\;&\exists\,Q_1,\ldots,Q_K\in\mathcal P,\\
&\forall P\in\mathcal P,\;
\min_{1\le j\le K} h(P,Q_j)\le r
\bigr\}.
\end{aligned}
\]

Here $N(r,\mathcal P,h)$ is the minimum number of Hellinger balls of radius $r$ needed to cover $\mathcal P$. It therefore measures the effective size of a potentially infinite target family at resolution $r$.

\begin{definition}[Target complexity]
\label{def:target-complexity}
For $\epsilon>0$, define
\[
C_{\mathrm{tar}}(\epsilon)
=
\sup_{\eta\in\mathcal H}
\log N\!\left(
c_0\sqrt{S_{\mathcal M}(\epsilon)},
\mathcal P_\eta,h
\right),
\]
where \(c_0>0\) is a sufficiently small universal constant.
\end{definition}

Hence $C_{\mathrm{tar}}(\epsilon)$ measures the effective number of target configurations that must be distinguished at accuracy $\epsilon$. Scalar or low-dimensional targets typically have modest complexity, whereas function-valued targets may require substantially larger covering sets.

\paragraph{Nuisance complexity.}
The preceding discussion treats $\eta$ as fixed, but in most observational causal problems nuisance quantities such as propensity scores, outcome regressions, and first-stage regressions must be estimated. Replacing $\eta$ by an estimate $\widehat\eta$ perturbs the observational law and can make competing causal targets harder to distinguish.

We separate this effect into two components. First, let $d$ be a metric on the nuisance space $\mathcal H$ and suppose an estimator $\widehat\eta$ satisfies, uniformly over $m\in\mathcal M$,
\[
P_m^{\otimes n}\{d(\eta,\widehat\eta)>r_n(\delta)\}\le\delta
\]
where $r_n(\delta)$ is a high-probability nuisance-estimation rate determined by the chosen estimator and nuisance class. Second, suppose the observational law is stable to nuisance perturbations:
\[
h(P_{\theta,\eta},P_{\theta,\eta'})
\le
L_{\mathrm{nuis}}\,d(\eta,\eta'),
\]
where $L_{\mathrm{nuis}}>0$ is a Lipschitz-type sensitivity constant converting nuisance error into observational-distribution error. In applications, $r_n(\delta)$ is obtained from finite-sample guarantees for the nuisance estimator, while $L_{\mathrm{nuis}}$ is derived from the model’s sensitivity to nuisance perturbations. We illustrate both quantities explicitly in Sections~3--5.

Combining the two conditions gives, with probability at least $1-\delta$, $h(P_{\theta,\eta},P_{\theta,\widehat\eta})\le L_{\mathrm{nuis}}r_n(\delta)$. Thus nuisance difficulty is governed by $L_{\mathrm{nuis}}r_n(\delta)$, which captures both how accurately the nuisance can be learned and how strongly nuisance error perturbs the observational law.
To preserve an $\epsilon$-scale causal signal, this perturbation must remain below the corresponding identification scale. A sufficient condition is
\[
L_{\mathrm{nuis}}r_n(\delta)
\lesssim
\sqrt{S_{\mathcal M}(\epsilon)}.
\]
By the triangle inequality for Hellinger distance, this ensures that nuisance estimation does not eliminate the observational separation between $\epsilon$-separated targets.
One important point is that nuisance dimension alone is not decisive: a high-dimensional nuisance may be benign if it is accurately estimated and weakly coupled to the observational law, while a low-dimensional nuisance may still dominate the sample complexity when estimation is difficult or sensitivity is high.

\subsection{A General PACC Sufficient Bound}
\label{sec:upper}

With the three factors above, we present a general sufficient bound.

\begin{theorem}[General PACC sufficient bound]
Fix $\epsilon>0$ and $0<\delta\le 1/8$. Assume that the parametrization is variation independent (every $(\theta,\eta)\in\Theta\times\mathcal H$ indexes a model in $\mathcal M$); that the nuisance estimator takes values in $\mathcal H$ and satisfies the rate and stability conditions above; that $S_{\mathcal M}(\epsilon)>0$ and $C_{\mathrm{tar}}(\epsilon)<\infty$; and that the measurability and rate-regularity conditions of Appendix~C hold. Then there are universal constants $C,c_*>0$ such that an $(\epsilon,\delta)$-PACC guarantee holds at every $n$ satisfying
\begin{equation}
n
\geq
C\frac{
C_{\mathrm{tar}}(\epsilon)+\log(1/\delta)
}{
S_{\mathcal M}(\epsilon)
}
\label{eq:pacc-target}
\end{equation}
and
\begin{equation}
L_{\mathrm{nuis}}r_n(\delta)
\leq
c_*\sqrt{S_{\mathcal M}(\epsilon)}.
\label{eq:pacc-nuis}
\end{equation}
\end{theorem}

The theorem gives a general sufficient principle rather than a method-specific optimal rate. Sharper rates may follow when a particular causal method admits additional structure.

To express the nuisance condition directly as a sample-size requirement, if $n\mapsto r_n(\delta)$ is nonincreasing, define
\begin{equation}
\begin{aligned}
n_{\mathrm{nuis}}(\epsilon,\delta)
=
\inf\bigl\{n\ge2:\;&
L_{\mathrm{nuis}}r_n(\delta)\\
&\le c_*\sqrt{S_{\mathcal M}(\epsilon)}
\bigr\},
\end{aligned}
\label{eq:n-nuis}
\end{equation}
where $c_*>0$ is a fixed universal constant implicit in Eq.~\eqref{eq:pacc-nuis}. Then, for $\delta\le1/8$,
\begin{equation}
\begin{aligned}
n_{\mathrm{PACC}}^*(\epsilon,\delta)
&\lesssim
\max\Biggl\{
\frac{
C_{\mathrm{tar}}(\epsilon)+\log(1/\delta)
}{
S_{\mathcal M}(\epsilon)
},
\,n_{\mathrm{nuis}}(\epsilon,\delta)
\Biggr\}.
\end{aligned}
\label{eq:pacc-summary}
\end{equation}

The full proof is given in the appendix~C. Only the identification term is matched by the general lower bound, up to replacing
$S_{\mathcal M}(\epsilon)$ by $S_{\mathcal M}(2\epsilon)$ (these agree up to constants
whenever $S_{\mathcal M}(2\epsilon)\asymp S_{\mathcal M}(\epsilon)$, as in all three
examples below); the target and nuisance terms are sufficient conditions for general causal learning problems is substantially more difficult, and is beyond the scope of this work. Instead, Theorem~1 serves to locate which of the three factors limits a given method: the identification term is the floor set by Proposition~1, while a dominant target or nuisance term indicates where
method-specific structure can help. Sections~3--5 follow this pattern,
and Table~1 summarizes the results.

%% file: sections/04_sccs.tex
\section{Self-Controlled Case Series}
\label{sec:sccs}

\begin{table}[t]
\caption{PACC factors across methods. Italics give the corresponding
classical notion; constants are suppressed.}
\label{tab:summary}
\centering
\footnotesize
\setlength{\tabcolsep}{3.0pt}
\renewcommand{\arraystretch}{1.20}
\begin{tabular}{@{}lccc@{}}
\toprule
 & SCCS & Many-IV & CATE \\
\midrule

$S_{\mathcal M}(\epsilon)$
&
\begin{tabular}[t]{@{}c@{}}
$E[W(1-W)]\epsilon^2$\\[-2pt]
\emph{contrast}
\end{tabular}
&
\begin{tabular}[t]{@{}c@{}}
$\lambda^2\epsilon^2$\\[-2pt]
\emph{weak IV}
\end{tabular}
&
\begin{tabular}[t]{@{}c@{}}
$\alpha(1-\alpha)\epsilon^2$\\[-2pt]
\emph{overlap}
\end{tabular}
\\

$C_{\mathrm{tar}}(\epsilon)$
&
\begin{tabular}[t]{@{}c@{}}
$\log\frac{1}{\epsilon}$\\[-2pt]
\emph{scalar}
\end{tabular}
&
\begin{tabular}[t]{@{}c@{}}
$\log \frac{1}{\lambda\epsilon}$\\[-2pt]
\emph{scalar}
\end{tabular}
&
\begin{tabular}[t]{@{}c@{}}
entropy of $\mathcal F$\\[-2pt]
\emph{CATE class}
\end{tabular}
\\

Nuisance
&
\begin{tabular}[t]{@{}c@{}}
none\\[-2pt]
\emph{self-controlled}
\end{tabular}
&
\begin{tabular}[t]{@{}c@{}}
$\pi\in\mathbb R^k$\\[-2pt]
\emph{many IVs}
\end{tabular}
&
\begin{tabular}[t]{@{}c@{}}
$(e_0,m_0)$\\[-2pt]
\emph{orthogonality}
\end{tabular}
\\

\midrule
Bottleneck
& identification
& nuisance
&
\begin{tabular}[t]{@{}c@{}}
nuisance,\\[-2pt]
target
\end{tabular}
\\
\bottomrule
\end{tabular}
\end{table}

We begin with the self-controlled case series (SCCS), a design for transient exposures and acute events in which each individual serves as their own control
\citep{farrington1995relative}.
SCCS compares event incidence during exposed risk periods with incidence during unexposed periods for the same person, so time-invariant individual characteristics cancel from the conditional comparison. This makes SCCS a clean first PACC example.

We consider the simplest single-event-per-case setting, which isolates the basic SCCS mechanism and is natural for rare event applications such as vaccine safety
\citep{farrington1995relative,whitaker2006tutorial}.
For case $i$, let $T_i$ denote the total observation time and $R_i$ the duration of an exposed risk period. Define $w_i:=R_i/T_i\in(0,1)$, the fraction of that case's observation period spent at risk. 

Let $x_i(t)\in\{0,1\}$ indicate whether time $t$ lies in the risk period, and consider the event intensity $\nu_i\exp\{\beta x_i(t)\}$, where $\nu_i>0$ is the individual's time-invariant baseline event rate and $\beta$ is the log relative incidence. Thus $e^\beta$ is the exposed-versus-unexposed relative incidence. The causal interpretation additionally requires that events do not alter subsequent exposure or observation and that relevant time-varying confounding is absent or appropriately modeled
\citep{whitaker2006tutorial,petersen2016self,etievant2026causal}.

Let $Y_i=1$ if the case's single event occurs during the risk period and $Y_i=0$ otherwise. Conditioning on the fact that one event occurred removes the common factor $\nu_i$, yielding
\begin{equation}
\begin{aligned}
Y_i\mid w_i
&\sim
\operatorname{Bernoulli}\!\left(p_\beta(w_i)\right),\\
\text{where  }p_\beta(w_i)
&=
\frac{w_ie^\beta}{1-w_i+w_ie^\beta}.
\end{aligned}
\label{eq:sccs-bernoulli}
\end{equation}
Thus SCCS reduces to a one-parameter conditional likelihood for $\beta$. We regard $w_1,\ldots,w_n$ as i.i.d. realizations of a random variable $W$, whose distribution does not depend on $\beta$. The conditional likelihood does not require modeling the distribution of $W$, while the finite-sample bound in Theorem~2 depends on it through the exposure-contrast quantity $\mathbb E[W(1-W)]$.

\paragraph{PACC sample complexity.}
We restrict $\beta$ to $|\beta|\le B<\infty$, ruling out arbitrarily extreme relative incidences and allowing uniform control over the parameter space. Since
$p_\beta(w)\{1-p_\beta(w)\}
=
w(1-w)e^\beta/(1-w+we^\beta)^2$
and $1-w+we^\beta$ lies between $1$ and $e^\beta$, it follows that
\begin{equation}
e^{-B}w(1-w)
\le
p_\beta(w)\{1-p_\beta(w)\}
\le
e^B w(1-w).
\label{eq:sccs-comparable}
\end{equation}
Thus the curvature of the conditional likelihood is comparable, up to constants depending only on \(B\), to the within-person exposure contrast \(w(1-w)\). Averaging over the case-level distribution of \(W\), and combining this with the Bernoulli Hellinger bound in Appendix~D, the PACC identification strength satisfies
\[
S_{\mathrm{SCCS}}(\epsilon)
\asymp_B
\mathbb E[W(1-W)]\,\epsilon^2,
\]
where \(\asymp_B\) denotes equality up to positive multiplicative constants depending only on \(B\).

\begin{theorem}[PACC sample complexity for canonical SCCS]
\label{thm:sccs-pacc}
Suppose $|\beta|\le B<\infty$ and $\mathbb{E}[W(1-W)]>0$. For sufficiently small $\epsilon>0$ and $0<\delta\le1/8$,
\begin{equation}
n_{\mathrm{PACC}}^*(\epsilon,\delta)
\asymp_B
\frac{\log(1/\delta)}
{\mathbb{E}[W(1-W)]\,\epsilon^2}
.
\label{eq:sccs-rate}
\end{equation}
\end{theorem}

For SCCS, the target complexity is $C_{\mathrm{tar}}(\epsilon)\asymp_B\log(1/\epsilon)$:
the Hellinger cover radius $c_0\sqrt{S_{\mathrm{SCCS}}(\epsilon)}$ and the separation
$h(P_\beta,P_{\beta'})$ both scale with $\sqrt{\mathbb{E}[W(1-W)]}$, so a grid on
$[-B,B]$ with spacing of order $\epsilon$ suffices. Theorem 1 would therefore add a $\log(1/\epsilon)$ term to the numerator; the direct conditional-MLE analysis in Theorem~2 removes it,
so identification strength is the only bottleneck.

\paragraph{Relation to classical SCCS theory.}
Classical SCCS sample-size formulas rely on large-sample approximations to tests of a prespecified relative incidence, such as the signed root likelihood ratio \citep{whitaker2006tutorial,musonda2006sample}; for a small log relative incidence $\beta_1$, the required number of cases scales to leading order as $(z_{1-\alpha/2}+z_{1-\gamma})^2/\{\mathbb E[W(1-W)]\beta_1^2\}$. Theorem~2 recovers this scaling, with $\epsilon$ in place of $\beta_1$ and $\log(1/\delta)$ in place of the squared normal quantiles, and sharpens it in two ways. First, the lower bound holds for every learner, not only for a particular test: the exposure contrast $\mathbb E[W(1-W)]$ is the intrinsic information limit of the design, and the conditional MLE attains it up to constants depending only on $B$. Second, the guarantee is non-asymptotic and uniform over all admissible effects $\beta\in[-B,B]$, rather than an approximation around a single prespecified alternative. SCCS thus serves as a calibration of the framework: in a setting with a well-understood classical answer, PACC identifies the same governing quantity and turns it into a two-sided finite-sample statement.

%% file: sections/05_IV.tex
\section{Instrumental Variables with Many Instruments}
\label{sec:iv}

Instrumental variables (IV) identify causal effects by isolating variation in treatment that is independent of unobserved confounding. In many applications, the analyst has not one instrument but a potentially large collection. Examples include Mendelian randomization with many genetic variants \citep{davies2015many} and judge-assignment designs that exploit variation across many decision makers \citep{dobbie2018effects}. This setting illustrates two distinct sources of difficulty in PACC learning: weak instruments make different causal effects harder to distinguish, whereas many instruments make the first-stage relationship harder to learn.

\paragraph{Model setup.}
Consider data $O=(Z,X,Y)$, where $Z\in\mathbb{R}^k$ is the instrument vector, $X$ is the treatment, and $Y$ is the outcome. We consider the linear IV model
\begin{equation}
X=Z^\top\pi+V,
\qquad
Y=\beta X+U,
\label{eq:iv-model}
\end{equation}
where the target is the scalar causal effect $\beta$. The errors $U$ and $V$ may be correlated, allowing for unobserved confounding.
We impose the standard IV assumptions: instrument exogeneity, $Z\perp(U,V)$; relevance, $\mathbb{E}[ZX]\neq 0$; and exclusion, meaning that $Z$ affects $Y$ only through $X$. These conditions imply
$\mathbb{E}[ZY]=\beta\mathbb{E}[ZX]$
so $\beta$ is identified whenever $\mathbb{E}[ZX]\neq0$.

We work with a Gaussian IV model: $Z\sim N(0,I_k)$, and $(U,V)$ is jointly Gaussian and independent of $Z$, with covariance eigenvalues uniformly bounded above and away from zero. We further assume $|\beta|$ and $\|\pi\|$ are uniformly bounded and, crucially,
\[
\|\pi\|\ge\lambda>0.
\]
Thus, $\lambda$ is the weakest first-stage strength allowed in the model class. 

\paragraph{General PACC bound.}
For the Gaussian IV model, the causal identification strength satisfies $S_{\mathrm{IV}}(\epsilon)\asymp\lambda^2\epsilon^2$, so identification weakens as the minimum first-stage strength $\lambda$ decreases; see Appendix~E for the Hellinger calculation.

Since $\beta$ is scalar and uniformly bounded, $C_{\mathrm{tar}}(\epsilon)=O(\log(1/(\lambda\epsilon)))$. The growing nuisance is the $k$-dimensional first-stage coefficient $\pi$, while the residual covariance is fixed-dimensional. Standard Gaussian regression concentration gives $\|\widehat\pi-\pi\|\lesssim\sqrt{(k+\log(1/\delta))/n}$ with probability at least $1-\delta$, hence
\[r_n(\delta)\lesssim\sqrt{(k+\log(1/\delta))/n}.\] Under the boundedness assumptions, $L_{\mathrm{nuis}}=O(1)$.

Applying Theorem~1 to the variation-independent reduced-form class of Appendix~E, the identification--target term requires 
\[n\gtrsim\{\log(1/(\lambda\epsilon))+\log(1/\delta)\}/(\lambda^2\epsilon^2),\]
while the nuisance condition \[L_{\mathrm{nuis}}r_n(\delta)\lesssim\sqrt{S_{\mathrm{IV}}(\epsilon)}\asymp\lambda\epsilon\]
requires $n\gtrsim\{k+\log(1/\delta)\}/(\lambda^2\epsilon^2)$. Therefore,
\begin{equation}
n^*_{\mathrm{PACC}}(\epsilon,\delta)
\lesssim
\frac{
k+\log(1/(\lambda\epsilon))+\log(1/\delta)
}{
\lambda^2\epsilon^2
}.
\label{eq:iv-general-bound}
\end{equation}
Thus $\lambda$ controls identification strength, whereas $k$ enters through nuisance learning. This generic sufficient bound need not be sharp; we next exploit the IV estimating equation to obtain a tighter guarantee.

\paragraph{$J$-fold cross-fitted IV.}
With many instruments, ordinary 2SLS can exhibit finite-sample bias from using the same observations to estimate both the first stage and the structural effect \citep{bekker1994,angrist1999jive}. We therefore use fixed-$J$-fold cross-fitting. Partition the sample into folds $I_1,\ldots,I_J$ and, for each $i\in I_j$, estimate $\pi$ by OLS using observations outside $I_j$, denoted $\widehat\pi^{(-j)}$, and set $\widehat f_i=Z_i^\top\widehat\pi^{(-j)}$. Define
\begin{equation}
\widehat\beta_{\mathrm{CF}}
=
\Pi_{[-B,B]}
\left(
\frac{\sum_{i=1}^n \widehat f_iY_i}
{\sum_{i=1}^n \widehat f_iX_i}
\right),
\label{eq:cfiv-estimator}
\end{equation}
where $\Pi_{[-B,B]}$ denotes projection onto $[-B,B]$. Thus each observation is paired with a first-stage prediction estimated without using its own fold.

\begin{theorem}[PACC bound for many instruments]
\label{thm:iv-pacc}
Under the conditions above, for fixed $J\ge2$, $0<\epsilon\le\min(1,B/3)$, and $0<\delta\le1/8$,
\begin{equation}
n^*_{\mathrm{PACC}}(\epsilon,\delta)
\lesssim
k
+
\frac{\sqrt{k\log(1/\delta)}}{\lambda^2\epsilon}
+
\frac{\log(1/\delta)}{\lambda^2\epsilon^2}.
\label{eq:iv-crossfit-rate}
\end{equation}
Moreover, if in addition $\epsilon\le\epsilon_0$, where $\epsilon_0>0$ depends only on the model-class bounds, every learner satisfies
\[
n^*_{\mathrm{PACC}}(\epsilon,\delta)
\gtrsim
\frac{\log(1/\delta)}{\lambda^2\epsilon^2}.
\]
The hidden constants depend only on $J$ and the fixed model-class bounds.
\end{theorem}

The three terms in \eqref{eq:iv-crossfit-rate} capture first-stage estimation, its interaction with weak identification, and the intrinsic identification cost. Compared with the generic bound \eqref{eq:iv-general-bound}, the IV-specific analysis replaces the $k/(\lambda^2\epsilon^2)$ dependence by an additive $k$ term and the smaller interaction $\sqrt{k\log(1/\delta)}/(\lambda^2\epsilon)$. When these first two terms are dominated by the final term, the upper and lower bounds match up to constants. Full proofs are given in Appendix~E.

\paragraph{Relation to classical IV theory.}
Classical theory measures instrument strength by the concentration parameter, of order $n\|\pi\|^2$ here \citep{staiger1994instrumental}, and shows that 2SLS is biased when $k$ is large relative to $n$ \citep{bekker1994,donald2001choosing}; with many weak instruments, consistency requires this parameter to grow faster than $\sqrt{k}$ \citep{chao2005consistent}. Theorem~\ref{thm:iv-pacc} gives a non-asymptotic, class-uniform counterpart: $\lambda$ enters through $S_{\mathrm{IV}}(\epsilon)\asymp\lambda^2\epsilon^2$, which lower-bounds every learner, while the interaction term $\sqrt{k\log(1/\delta)}/(\lambda^2\epsilon)$ mirrors the $\sqrt{k}$ condition.

%% file: sections/06_ps.tex
\section{Propensity Scores and Heterogeneous Effects}
\label{sec:ps}

SCCS conditions its nuisance away, while IV retains nuisance estimation but has a scalar target. We now consider CATE estimation under unconfoundedness and overlap, where all three PACC components are active: the target is function-valued, propensity and outcome regressions must be learned, and weak overlap reduces observational separation between distinct treatment effects.

\paragraph{Model setup.}
Each unit contributes $O=(X,A,Y)$, where $X$ denotes observed covariates and $A\in\{0,1\}$ is a binary treatment. We assume consistency, $Y=Y(A)$; unconfoundedness,
$(Y(0),Y(1))\perp A\mid X$; and overlap, $\alpha\le e_0(X)\le1-\alpha$ almost surely for a fixed
$\alpha\in(0,1/2]$, where $e_0(x)=\Pr(A=1\mid X=x)$. The target is the CATE
$\tau_0(x)=\mathbb E[Y(1)-Y(0)\mid X=x]$, with error
$\rho(\tau,\tau')=\|\tau-\tau'\|_{L_2(P_X)}$.

Let $m_0(x)=\mathbb E[Y\mid X=x]$. Then
$\mathbb E[Y-m_0(X)\mid X,A]=\{A-e_0(X)\}\tau_0(X)$, which motivates the
population R-loss
\[
  R(f)=\mathbb E\big[\{Y-m_0(X)-(A-e_0(X))f(X)\}^2\big]
\]
\citep{robinson1988root,nie2021quasi}. Its excess risk is
\begin{align*}
  R(f)-R(\tau_0)&=\mathbb E\big[e_0(1-e_0)(f-\tau_0)^2\big]\\
  &\ge\alpha(1-\alpha)\|f-\tau_0\|_{L_2(P_X)}^2.
\end{align*}
Thus $e_0(1-e_0)=\mathrm{Var}(A\mid X)$ is the
curvature available for learning heterogeneous effects, and weak overlap
reduces it.

\paragraph{PACC components.}
For explicit bounds we use a Gaussian submodel $\mathcal M_{\mathrm{PS}}$.
Let $\mathcal F\subseteq\{f:\|f\|_\infty\le B\}$ be the CATE class and
$\mathcal H=\{(e,m):\alpha\le e\le1-\alpha,\ \|m\|_\infty\le B\}$. For
$\tau\in\mathcal F$ and $\eta=(e,m)\in\mathcal H$, $P_{\tau,\eta}$ is the law
with $A\mid X\sim\mathrm{Bernoulli}(e(X))$ and
\[
  Y\mid X,A\sim N\big(m(X)+\{A-e(X)\}\tau(X),\,\sigma^2\big),
\]
with $P_X$ and $\sigma^2>0$ fixed.

\begin{proposition}[PACC components for CATE]\label{prop:ps}
For $\mathcal M_{\mathrm{PS}}$ and $0<\epsilon\le2B$:

\emph{(i) Identification strength.}
\[S_{\mathcal M_{\mathrm{PS}}}(\epsilon)\gtrsim\alpha(1-\alpha)\epsilon^2/(\sigma^2+B^2),\]
and $S_{\mathcal M_{\mathrm{PS}}}(\epsilon)\lesssim\alpha(1-\alpha)\epsilon^2/\sigma^2$
if the constant propensity $\alpha$ is admissible and $\mathcal F$ contains
$\tau,\tau'$ with $\epsilon<\|\tau-\tau'\|_{L_2(P_X)}\le2\epsilon$;

\emph{(ii) Target complexity.}
\[
  C_{\mathrm{tar}}(\epsilon)\le\log N\Big(c_0\tfrac{\sigma\sqrt{\alpha(1-\alpha)}}{\sqrt{\sigma^2+B^2}}\,\epsilon,\ \mathcal F,\ L_2(P_X)\Big);
\]

\emph{(iii) Nuisance sensitivity.} for $\tau\in\mathcal F$ and $\eta=(e,m),\eta'=(e',m')\in\mathcal H$,
\begin{align*}
  h(P_{\tau,\eta},P_{\tau,\eta'})&\lesssim\big(\alpha^{-1/2}+B/\sigma\big)\|e-e'\|_{L_2(P_X)}\\
  &\quad+\|m-m'\|_{L_2(P_X)}/\sigma .
\end{align*}
\end{proposition}

\paragraph{Generic PACC bound and orthogonal refinement.}
Proposition~\ref{prop:ps} makes the three sources of difficulty explicit:
overlap controls observational separation, $\mathcal F$ contributes target
complexity, and $(e_0,m_0)$ must be estimated. Theorem~1 then requires
$n\gtrsim(\sigma^2+B^2)\{C_{\mathrm{tar}}(\epsilon)+\log(1/\delta)\}/\{\alpha(1-\alpha)\epsilon^2\}$,
and, since $L_{\mathrm{nuis}}\asymp\alpha^{-1/2}$ and
$\sqrt{S_{\mathcal M_{\mathrm{PS}}}(\epsilon)}\asymp\sqrt{\alpha}\,\epsilon$,
each nuisance error $\|\hat e-e_0\|_{L_2(P_X)}$, $\|\hat m-m_0\|_{L_2(P_X)}$
must be at most a small multiple of $\alpha\epsilon$
(see Appendix~F for details). Orthogonality relaxes this to a condition on
products of nuisance errors, so each error need only be of order
$\sqrt{\alpha\epsilon}$.

As in Section~\ref{sec:iv}, we use $J$-fold cross-fitting, which is standard
for orthogonal learners \citep{chernozhukov2018double,nie2021quasi}: with
folds $I_1,\dots,I_J$, let $\hat e^{(-j)},\hat m^{(-j)}$ be trained outside
fold $j$, and let $\hat\tau$ minimize the cross-fitted empirical version of
$R$ over $\mathcal F$:
\begin{align*}
  \hat\tau\in\arg\min_{f\in\mathcal F}\sum_{j=1}^J\sum_{i\in I_j}
  \Big[&Y_i-\hat m^{(-j)}(X_i)\\
  &-\{A_i-\hat e^{(-j)}(X_i)\}f(X_i)\Big]^2 .
\end{align*}
Nuisance error then enters only through the second-order remainder
\[
  b_j=(\hat e^{(-j)}-e_0)(\hat m^{(-j)}-m_0)-\tau_0(\hat e^{(-j)}-e_0)^2.
\]

\begin{theorem}[Orthogonal PACC bound for CATE]\label{thm:ps}
Let $0<\epsilon\le2B$ and $0<\delta\le1/8$, and suppose
Assumption~\ref{asm:F-orth} in Appendix~F holds; in particular,
$\max_j\|b_j\|_{L_2(P_X)}\le c\,\alpha(1-\alpha)\epsilon$ with probability at
least $1-\delta/2$, for a sufficiently small universal constant $c>0$. If
\begin{align*}
  n\gtrsim_J\frac{\sigma^2+B^2}{\alpha(1-\alpha)\epsilon^2}
  \Big\{&\log N\Big(c\,\tfrac{\alpha(1-\alpha)\epsilon^2}{\sigma+B},\mathcal F,\|\cdot\|_\infty\Big)\\
  &+\log\tfrac1\delta\Big\},
\end{align*}
then $\hat\tau$ satisfies the $(\epsilon,\delta)$-PACC guarantee.
\end{theorem}

The overlap dependence is unavoidable: if the constant propensity $\alpha$
is admissible and $\mathcal F$ contains $\tau_\pm$ with
$2\epsilon<\|\tau_+-\tau_-\|_{L_2(P_X)}\le4\epsilon$, then for small
$\epsilon$,
$n^*_{\mathrm{PACC}}(\epsilon,\delta)\gtrsim\sigma^2\log(1/\delta)/\{\alpha(1-\alpha)\epsilon^2\}$
(Proposition~1). Orthogonality thus weakens the nuisance
requirement but not the identification scale. If $\alpha$ is not bounded away
from zero over the class, the identification strength vanishes, so pointwise
identification need not yield PACC learnability.

\paragraph{Relation to CATE theory.}
Orthogonal, cross-fitted CATE learners are well studied: the R-learner
attains oracle-type rates when the nuisance errors are $o(n^{-1/4})$
\citep{nie2021quasi}, and orthogonal statistical learning yields the same
second-order nuisance dependence for general losses
\citep{foster2023orthogonal}. These guarantees are stated for a fixed
data-generating process, with overlap treated as a constant.
Theorem~\ref{thm:ps} recovers this second-order structure as a
class-uniform, non-asymptotic sample-size requirement in which overlap
enters explicitly through $\alpha(1-\alpha)$. The matching lower bound shows
that the factor $1/\{\alpha(1-\alpha)\epsilon^2\}$ is required by every
learner: it reflects identification strength, not a feature of the
R-learner, so orthogonality can weaken only the nuisance requirement. \citet{khan2010irregular} show that weak overlap can make effect estimation
irregular, with slower than $\sqrt n$ rates. Our bounds make this cost
explicit: the required sample size grows as $1/\alpha$, and uniform
learnability fails when $\alpha$ is not bounded away from zero. Our dependence on $\mathcal F$ is not optimal: Theorem~\ref{thm:ps} uses
sup-norm entropy at radius of order $\epsilon^2$, which is coarser than the
localized $L_2$ analyses behind rate-optimal CATE estimators. This suffices
for a class-uniform guarantee and isolates the roles of overlap and
nuisance error; sharper target complexity could follow from localization.

%% file: sections/03_relatedwork.tex
\section{Related Work}
\label{sec:related}

\paragraph{PAC-style and finite-sample causal learning.}
Recent work establishes PAC guarantees for specific causal tasks, including learning interventional distributions \citep{bhattacharyya2022efficient}, covariate adjustment \citep{choo2025probably}, ATE estimation \citep{yin2025sampling}, and causal trees \citep{tadepalli2021pac}. In contrast, PACC provides a method-agnostic framework for causal learnability, characterizing finite-sample guarantees and fundamental limitations at the level of the causal problem $(\mathcal M,\phi,\rho)$ rather than a specific method. 

Traditional sample-size analyses \citep{shooksa2022power,liu2026sample,musonda2006sample} typically target prespecified effects, significance levels, and power. PACC instead provides uniform accuracy--confidence guarantees. We compare our PACC bounds with classical results at the end of the SCCS, IV, and CATE sections.

\paragraph{Identification strength, statistical complexity, and nuisance learning.}
The tools underlying PACC are classical. Two-point lower bounds based on Hellinger or total variation distance are standard \citep{lecam1973convergence,tsybakov2009introduction}. For estimating a functional, the modulus of continuity with respect to Hellinger distance governs minimax difficulty \citep{donoho1991geometrizing}; $S_{\mathcal M}(\epsilon)$ is an inverse modulus of this kind. Estimators built from tests over Hellinger covers link metric entropy to finite-sample accuracy \citep{lecam1973convergence,birge1983approximation,yang1999information}, and our general sufficient bound follows this route. In causal inference, weak identification can preclude uniformly informative inference \citep{gleser1987nonexistence,dufour1997some}, poor overlap can make effect estimation irregular \citep{khan2010irregular}, and orthogonal, cross-fitted procedures reduce first-order sensitivity to nuisance estimation \citep{chernozhukov2018double,foster2023orthogonal}. These results are typically developed separately, each for a particular phenomenon or estimator. PACC places them within one finite-sample question about a causal model class. The causal assumptions encoded in $\mathcal M$ determine the identification strength, and the weak-instrument and weak-overlap failures above become instances of $S_{\mathcal M}(\epsilon)$ vanishing over the class. This yields an estimator-independent lower bound and a sufficient bound that separates identification from target and nuisance learning. As a result, classical quantities such as exposure contrast, instrument strength, and overlap can be compared on a common scale.

%% file: sections/07_discussion.tex
\section{Conclusion and Discussion}

We introduced PAC-Causal (PACC) learnability as a finite-sample counterpart to causal identification: a causal target is PACC-learnable over a model class if it can be estimated to accuracy $\epsilon$ with probability at least $1-\delta$, uniformly over the class. Our general sufficient bound breaks sample complexity into three factors: identification strength, target complexity, and nuisance complexity. In SCCS, many-IV, and CATE, the same three factors appear, each through a different statistical mechanism.

This perspective goes beyond identification by making the sample cost of weak identification explicit, alongside the costs of target and nuisance complexity. The resulting bounds also inform study design by revealing which factor drives the required sample size and whether collecting more data, strengthening identification, or restricting the model class would be most effective.

The general theorem provides only sufficient conditions, as its generality comes at the cost of tightness. On the lower-bound side, we show that weak observational separation imposes a cost that no learner can avoid, consistent with the identification-strength factor in the upper bound. Whether the target- and nuisance-complexity factors are also unavoidable remains open. For example, the dependence on the number of instruments in the many-IV bound reflects the cost of the estimator we analyze, and we do not claim that this dependence is minimax optimal.

Several directions are worth pursuing. One is to identify conditions under which PACC learnability admits a necessary-and-sufficient characterization, including lower bounds that match the target- and nuisance-complexity terms. Another is to extend the framework to settings such as difference-in-differences, regression discontinuity, and longitudinal treatment regimes, where identification strength may take different forms. It would also be useful to develop data-dependent bounds that can be estimated from the observed sample. PACC provides a common PAC-style language for studying finite-sample guarantees across causal methods that are typically analyzed separately.

\newpage
\section*{AI Use Statement}
In this work, we used generative AI tools to assist with checking and drafting parts of the
mathematical proofs and with verifying calculations; every proof was checked line by line by the
authors. We did not use generative AI tools for theoretical or conceptual development of the
framework, formulating hypotheses, or research methodology. Experiments, method implementation,
synthetic-data generation, dataset cleaning, qualitative analysis, results interpretation and
translation are not applicable to this work. Additionally, we used generative AI tools for
brainstorming, language editing, manuscript organization and reference formatting; all citations
were checked against the original sources. We have reviewed all AI-assisted work and take
responsibility for the final content of this work.

%% file: sections/09_appendix.tex
\appendix
 
\section{Proof of the Weak-IV Counterexample}
\label{app:weakiv}
 
Fix $0<\epsilon<B$ and $\delta\in(0,1/2)$, and choose $\Delta$ with $2\epsilon<\Delta\le 2B$. Set $\beta_\pm=\pm\Delta/2$, so that $|\beta_\pm|\le B$. Let $Z,U_X,\xi$ be independent standard Gaussians and, for $s\in\{+,-\}$, define
\[
U_Y^{(s)}=-\beta_sU_X+\xi,\quad X=\pi Z+U_X,\quad Y=\beta_sX+U_Y^{(s)}.
\]
Then $Z$ is independent of $(U_X,U_Y^{(s)})$, whose covariance matrix
\[
\begin{pmatrix}
1 & -\beta_s\\
-\beta_s & 1+\beta_s^2
\end{pmatrix}
\]
has determinant one, so both constructions are admissible Gaussian IV models. Moreover, $Y=\beta_s\pi Z+\xi$, and since $\xi\perp U_X$,
\[
(X,Y)\mid Z\ \sim\ N\!\left(\begin{pmatrix}\pi Z\\ \beta_s\pi Z\end{pmatrix},\,I_2\right).
\]
Thus the two models share the law of $(Z,X)$, and their conditional laws of $(X,Y)\mid Z$ differ only in the mean of $Y$, by $\Delta\pi Z$. The Gaussian affinity formula gives
\[
1-h^2(P_\pi^+,P_\pi^-)
=\mathbb E\exp\!\left(-\frac{\Delta^2\pi^2Z^2}{8}\right)
=\left(1+\frac{\Delta^2\pi^2}{4}\right)^{-1/2},
\]
hence $h^2(P_\pi^+,P_\pi^-)\le\Delta^2\pi^2/8$. By Hellinger tensorization and $\operatorname{TV}(P,Q)\le\sqrt2\,h(P,Q)$,
\[
\operatorname{TV}\!\left((P_\pi^+)^{\otimes n},(P_\pi^-)^{\otimes n}\right)\le\frac{|\pi|\Delta\sqrt n}{2}.
\]
Fix $n$ and suppose a learner $\widehat\beta=\mathcal A_n(S_n)$ were $\epsilon$-accurate with probability at least $1-\delta$ under both models. Since $|\beta_+-\beta_-|=\Delta>2\epsilon$, the event $A=\{|\widehat\beta-\beta_+|\le\epsilon\}$ satisfies $(P_\pi^+)^{\otimes n}(A)\ge1-\delta$ and $(P_\pi^-)^{\otimes n}(A)\le\delta$. Therefore
\[
\operatorname{TV}\!\left((P_\pi^+)^{\otimes n},(P_\pi^-)^{\otimes n}\right)\ge1-2\delta>0.
\]
Choosing $0<|\pi|\le1$ so small that $|\pi|\Delta\sqrt n/2<1-2\delta$ gives a contradiction. Since $n$ was arbitrary, no finite sample size yields a uniform $(\epsilon,\delta)$-PACC guarantee over the class.
 
For every fixed $\pi\neq0$, on the other hand, $\operatorname{Cov}(Z,X)=\pi$ and $\operatorname{Cov}(Z,Y)=\beta\pi$, so
\[
\beta=\frac{\operatorname{Cov}(Z,Y)}{\operatorname{Cov}(Z,X)}
\]
is identified and the Wald estimator is pointwise consistent. Hence pointwise identification and consistency do not imply PACC learnability when arbitrarily weak instruments are allowed.

\section{Proof of the Identification Barrier}
\label{app:id-barrier}
 
\begin{proof}[Proof of Proposition~1]
Fix $\epsilon>0$ and $0<\delta\le 1/8$. Suppose first that
$0<S_{\mathcal M}(2\epsilon)<1/2$. By the definition of
$S_{\mathcal M}(2\epsilon)$, there exist pairs $(m_j,m'_j)$ such that
\[
\rho\bigl(\phi(m_j),\phi(m'_j)\bigr)>2\epsilon,
\qquad
h^2(P_{m_j},P_{m'_j})
\to S_{\mathcal M}(2\epsilon).
\]
 
Suppose that $\widehat\phi_n$ satisfies the $(\epsilon,\delta)$-PACC
guarantee. Since the two targets are more than $2\epsilon$ apart,
their $\epsilon$-balls are disjoint, so $\widehat\phi_n$ induces a test
between $P_{m_j}^{\otimes n}$ and $P_{m'_j}^{\otimes n}$ whose sum of
errors is at most $2\delta$. Hence
\[
2\delta
\ge
1-\operatorname{TV}
\bigl(P_{m_j}^{\otimes n},P_{m'_j}^{\otimes n}\bigr)
\ge
\frac12
\bigl\{1-h^2(P_{m_j},P_{m'_j})\bigr\}^{2n},
\]
where the last inequality uses the relation between total variation
and Hellinger affinity together with Hellinger tensorization.
 
For all sufficiently large $j$,
$h^2(P_{m_j},P_{m'_j})<1/2$, and therefore
\[
n
\ge
\frac{\log(1/(4\delta))}
{4h^2(P_{m_j},P_{m'_j})}.
\]
Letting $j\to\infty$ and using $\delta\le1/8$ gives
\[
n
\gtrsim
\frac{\log(1/\delta)}
{S_{\mathcal M}(2\epsilon)}.
\]
Thus the same bound holds for
$n_{\mathrm{PACC}}^*(\epsilon,\delta)$.
 
If $S_{\mathcal M}(2\epsilon)=0$, choose pairs as above with
$h^2(P_{m_j},P_{m'_j})\to0$. For any fixed $n$,
\[
\operatorname{TV}
\bigl(P_{m_j}^{\otimes n},P_{m'_j}^{\otimes n}\bigr)
\le
\sqrt{2n\,h^2(P_{m_j},P_{m'_j})}
\longrightarrow 0.
\]
On the other hand, the $(\epsilon,\delta)$-PACC guarantee would imply
that this total variation distance is at least $1-2\delta>0$, a
contradiction. Hence no finite sample size provides an
$(\epsilon,\delta)$-PACC guarantee.
\end{proof}

\section{Proof of the General PACC Sufficient Bound}
\label{app:general}
 
We require two mild technical conditions beyond those stated in the main text. First, the Hellinger covers and the resulting finite testing rule are assumed measurable, so that the construction below defines a valid estimator. Second, because the proof uses sample splitting, we assume that reducing the nuisance-estimation sample size from \(n\) to \(\lfloor n/2\rfloor\) and the failure probability from \(\delta\) to \(\delta/2\) changes the nuisance rate by at most a constant factor:
\[
r_{\lfloor n/2\rfloor}(\delta/2)
\le C_r r_n(\delta)
\]
for some universal constant \(C_r<\infty\).
 
We also use a standard robust Hellinger testing result. There exist universal constants \(c_{\mathrm{rob}},c_{\mathrm{test}}>0\) such that, for any two reference distributions \(Q_0,Q_1\), one can construct a test whose error is at most
\[
\exp\{-c_{\mathrm{test}}N h^2(Q_0,Q_1)\}
\]
from \(N\) observations, even when the true distribution is not exactly \(Q_i\) but lies within Hellinger distance
\(c_{\mathrm{rob}}h(Q_0,Q_1)\) of \(Q_i\), for \(i=0,1\); see, e.g., \citet{lecam1973convergence,birge1983approximation}.
This robustness is needed because the true distribution is only approximated by a cover element after nuisance estimation.
 
\begin{proof}[Proof of Theorem~1]
Let the true model be \(P_{\theta_0,\eta_0}\), and split the sample into
\(n_1=\lfloor n/2\rfloor\) and \(n_2=n-n_1\) observations. Using the first
part, construct \(\widehat\eta=\widehat\eta_{n_1}\). By the nuisance-rate,
stability, and rate-regularity conditions, with probability at least
\(1-\delta/2\),
\[
h(P_{\theta_0,\eta_0},P_{\theta_0,\widehat\eta})
\le
C_r L_{\mathrm{nuis}} r_n(\delta)
\le
C_r c_*\sqrt{S_{\mathcal M}(\epsilon)}.
\]
 
Conditional on the first part of the sample, take a
\(c_0\sqrt{S_{\mathcal M}(\epsilon)}\)-cover
\[
Q_j=P_{\theta_j,\widehat\eta},
\qquad j=1,\ldots,K,
\]
of \(\mathcal P_{\widehat\eta}\), where \(c_0>0\) is a sufficiently small
universal constant and
\[
K\le \exp\{C_{\mathrm{tar}}(\epsilon)\}.
\]
Let \(Q_{j^*}\) satisfy
\[
h(P_{\theta_0,\widehat\eta},Q_{j^*})
\le
c_0\sqrt{S_{\mathcal M}(\epsilon)}.
\]
On the nuisance-estimation event,
\[
h(P_{\theta_0,\eta_0},Q_{j^*})
\le
(c_0+C_r c_*)\sqrt{S_{\mathcal M}(\epsilon)}.
\]
 
Using the second part of the sample, perform a binary robust test for each
pair \(Q_j,Q_k\) satisfying
\[
h(Q_j,Q_k)>
\frac12\sqrt{S_{\mathcal M}(\epsilon)},
\]
and declare one member of each tested pair the winner. Choose \(c_0\) and
\(c_*\) sufficiently small that
\[
c_0+C_r c_*
\le
\frac{c_{\mathrm{rob}}}{2}.
\]
Then, whenever \(Q_{j^*}\) is tested against another candidate, the robust
testing bound gives
\[
P_{\theta_0,\eta_0}^{\otimes n_2}
\{\text{\(Q_{j^*}\) loses the test}\}
\le
\exp\{-c n S_{\mathcal M}(\epsilon)\}
\]
for a universal constant \(c>0\). A union bound over at most \(K\)
competitors therefore yields
\[
P_{\theta_0,\eta_0}^{\otimes n}
\{\text{\(Q_{j^*}\) loses some test}\}
\le
\exp\left\{
C_{\mathrm{tar}}(\epsilon)+1
-c nS_{\mathcal M}(\epsilon)
\right\}.
\]
Hence this probability is at most \(\delta/2\) whenever
\[
n
\ge
C
\frac{C_{\mathrm{tar}}(\epsilon)+\log(1/\delta)}
{S_{\mathcal M}(\epsilon)}
\]
for a sufficiently large universal constant \(C\).
 
Return any candidate that loses none of its tests. With probability at least
\(1-\delta\), \(Q_{j^*}\) is undefeated. On this event, any returned
\(Q_{\widehat j}=P_{\widehat\theta,\widehat\eta}\) must satisfy
\[
h(Q_{\widehat j},Q_{j^*})
\le
\frac12\sqrt{S_{\mathcal M}(\epsilon)};
\]
otherwise \(Q_{\widehat j}\) and \(Q_{j^*}\) would have been tested, and
since \(Q_{j^*}\) is undefeated, \(Q_{\widehat j}\) could not be returned.
Consequently,
\[
h(P_{\widehat\theta,\widehat\eta},
P_{\theta_0,\widehat\eta})
\le
\left(\frac12+c_0\right)
\sqrt{S_{\mathcal M}(\epsilon)}
<
\sqrt{S_{\mathcal M}(\epsilon)}.
\]
 
By variation independence,
\((\widehat\theta,\widehat\eta)\) and
\((\theta_0,\widehat\eta)\) both index admissible models. The definition of
\(S_{\mathcal M}(\epsilon)\) therefore implies
\[
\rho(\widehat\theta,\theta_0)\le\epsilon.
\]
Thus the learner satisfies the \((\epsilon,\delta)\)-PACC guarantee uniformly
over the model class.
\end{proof}

\section{Proof for Self-Controlled Case Series}
\label{app:sccs}
 
\begin{proof}[Proof of Theorem~2]
Recall that
\[
p_\beta(w)=\frac{we^\beta}{1-w+we^\beta},
\qquad |\beta|\le B,
\]
and let \((W_i,Y_i)_{i=1}^n\) be i.i.d. observations from
\(P_\beta\), where
\(Y_i\mid W_i\sim\operatorname{Ber}(p_\beta(W_i))\).
 
We first verify the identification strength used in the
main text. For \(p,q\in[0,1]\),
\[
h^2(\operatorname{Ber}(p),\operatorname{Ber}(q))
=
\frac12(\sqrt p-\sqrt q)^2
+
\frac12(\sqrt{1-p}-\sqrt{1-q})^2.
\]
For \(|b|\le B\),
\(p_b(w)\asymp_B w\) and
\(1-p_b(w)\asymp_B1-w\). Moreover,
\[
\frac{d}{db}\sqrt{p_b(w)}
=
\frac12\sqrt{p_b(w)}\{1-p_b(w)\},
\]
and
\[
\frac{d}{db}\sqrt{1-p_b(w)}
=
-\frac12p_b(w)\sqrt{1-p_b(w)}.
\]
Thus, applying the mean-value theorem separately to the
two square-root terms, their squared increments satisfy
\[
\bigl(\sqrt{p_\beta(w)}-\sqrt{p_{\beta'}(w)}\bigr)^2
+
\bigl(\sqrt{1-p_\beta(w)}-\sqrt{1-p_{\beta'}(w)}\bigr)^2
\asymp_B
w(1-w)(\beta-\beta')^2.
\]
Hence
\[
h^2(P_\beta,P_{\beta'})
\asymp_B
\mathbb E[W(1-W)](\beta-\beta')^2,
\]
so
\[
S_{\mathrm{SCCS}}(\epsilon)
\asymp_B
\mathbb E[W(1-W)]\epsilon^2.
\]
The lower bound in Theorem~2 now follows from
Proposition~1.
For the upper bound, the conditional log-likelihood given $W_1,\ldots,W_n$ is
\[
\ell_n(b)
=
\sum_{i=1}^n
\left[
Y_i\log p_b(W_i)
+
(1-Y_i)\log\{1-p_b(W_i)\}
\right].
\]
Its derivative (the score) and negative second derivative
are
\[
\ell_n'(b)
=
\sum_{i=1}^n\{Y_i-p_b(W_i)\},
\]
and
\[
-\ell_n''(b)
=
\sum_{i=1}^n
p_b(W_i)\{1-p_b(W_i)\}.
\]
By Eq.~(6) in the main text,
\[
-\ell_n''(b)
\ge
e^{-B}\sum_{i=1}^n W_i(1-W_i).
\]
 
Write \(\kappa=\mathbb E[W(1-W)]\).
Since \(0\le W(1-W)\le1/4\) and
\(\operatorname{Var}(W(1-W))\le\kappa/4\),
Bernstein's inequality implies
\[
P\left\{
\sum_{i=1}^nW_i(1-W_i)
<
\frac{n\kappa}{2}
\right\}
\le
e^{-c n\kappa}.
\]
On the complementary event,
\(-\ell_n''(b)\ge e^{-B}n\kappa/2\) uniformly in
\(b\in[-B,B]\). By concavity of \(\ell_n\), if the
conditional MLE \(\widehat\beta\in\arg\max_{|b|\le B}\ell_n(b)\) satisfies
\(|\widehat\beta-\beta|>\epsilon\), then
\[
|\ell_n'(\beta)|
\ge
\frac12e^{-B}n\kappa\epsilon.
\]
 
Under the true parameter,
\(Y_i-p_\beta(W_i)\) has mean zero, is bounded by one,
and has variance
\[
\mathbb E\!\left[
p_\beta(W)\{1-p_\beta(W)\}
\right]
\le e^B\kappa.
\]
A second application of Bernstein's inequality therefore
gives, for sufficiently small \(\epsilon\),
\[
P_\beta\!\left(
|\widehat\beta-\beta|>\epsilon
\right)
\le
3\exp(-c_B n\kappa\epsilon^2).
\]
Consequently, the error probability is at most \(\delta\)
whenever
\[
n
\gtrsim_B
\frac{\log(1/\delta)}
{\kappa\epsilon^2}.
\]
Together with the lower bound and
\(\kappa=\mathbb E[W(1-W)]\), this proves
\[
n_{\mathrm{PACC}}^*(\epsilon,\delta)
\asymp_B
\frac{\log(1/\delta)}
{\mathbb E[W(1-W)]\epsilon^2}.
\]
\end{proof}

\section{Proofs for Gaussian Instrumental Variables}
\label{app:iv}
 
We prove the claims in Section~4 for the Gaussian IV model
\[
X=Z^\top\pi+V,
\qquad
Y=\beta X+U,
\]
where $Z\sim N(0,I_k)$ and $(U,V)$ is jointly Gaussian and
independent of $Z$. Substituting the first equation into the second gives
\[
Y=\beta Z^\top\pi+(\beta V+U).
\]
Hence, conditional on $Z$,
\[
\begin{pmatrix}X\\Y\end{pmatrix}
\Bigm| Z
\sim
N\left(
\begin{pmatrix}
Z^\top\pi\\
\beta Z^\top\pi
\end{pmatrix},
\Omega
\right),
\]
where $\Omega=\operatorname{Cov}(V,\beta V+U)$.
Under the assumptions of Section~4, the eigenvalues of $\Omega$
are uniformly bounded above and away from zero.
 
For the generic PACC bound, we enlarge the structural model to the
variation-independent reduced-form class
\[
\mathcal P_{\mathrm{IV}}
=
\left\{
P_{\beta,\pi,\Omega}:
|\beta|\le B,\;
\lambda\le\|\pi\|\le\Lambda,\;
c_\Omega I_2\preceq\Omega\preceq C_\Omega I_2
\right\}.
\]
This class contains the original structural model and allows
$(\beta,\pi,\Omega)$ to vary independently.
 
\paragraph{Identification strength.}
Consider two distributions
$P_{\beta,\pi,\Omega}$ and
$P_{\beta',\pi',\Omega'}$.
Their conditional Gaussian means differ by
\[
\begin{pmatrix}
Z^\top(\pi-\pi')\\
Z^\top(\beta\pi-\beta'\pi')
\end{pmatrix}.
\]
For the reduced-form Gaussian class above, the standard Gaussian
Hellinger formula and the uniform eigenvalue bounds imply
\[
h^2(P_{\beta,\pi,\Omega},P_{\beta',\pi',\Omega'})
\asymp
\|\pi-\pi'\|^2
+\|\beta\pi-\beta'\pi'\|^2
+\|\Omega-\Omega'\|_F^2,
\]
uniformly over the bounded parameter class, where
$\|\cdot\|_F$ denotes the Frobenius norm. The constants depend only
on the fixed covariance and parameter bounds. For models farther apart,
the corresponding lower bound is bounded away from zero, which is
sufficient for the small-$\epsilon$ analysis below.
 
Suppose $|\beta-\beta'|>\epsilon$. Since
\[
(\beta-\beta')\pi'
=
(\beta\pi-\beta'\pi')
-\beta(\pi-\pi'),
\]
the triangle inequality gives
\[
|\beta-\beta'|\,\|\pi'\|
\le
\|\beta\pi-\beta'\pi'\|
+
|\beta|\,\|\pi-\pi'\|.
\]
Using $\|\pi'\|\ge\lambda$ and $|\beta|\le B$,
\[
\lambda^2|\beta-\beta'|^2
\lesssim
\|\pi-\pi'\|^2
+
\|\beta\pi-\beta'\pi'\|^2.
\]
Thus, for sufficiently small $\epsilon$,
\[
S_{\mathrm{IV}}(\epsilon)
\gtrsim
\lambda^2\epsilon^2.
\]
 
For the reverse inequality, let $\pi=\pi'=\lambda (1,0,\dots,0)^\top\in\mathbb{R}^k$. Set $\beta=3\epsilon/2$ and $\beta'=-3\epsilon/2$; both are admissible since $\epsilon\le B/3$, and $|\beta-\beta'|=3\epsilon>2\epsilon$. To keep the reduced-form covariance identical, fix a covariance $\Sigma_0$ of $(U,V)$ in the interior of the structural covariance class of Section~4, and let
\[
(V,R)^\top\sim N(0,\Omega_0),
\]
where $\Omega_0$ is the covariance of $(V,U)^\top$ under $\Sigma_0$. For the model with effect $\beta$, define $U=R-\beta V$; for the model with effect $\beta'$, $U'=R-\beta' V$. Then
\[
(V,\beta V+U)=(V,R)\quad\text{and}\quad(V,\beta' V+U')=(V,R),
\]
so both models have reduced-form covariance $\Omega_0$. The structural covariances of $(U,V)$ and $(U',V)$ differ from $\Sigma_0$ by $O(\epsilon)$, so they remain in the structural class for $\epsilon\le\epsilon_0$, where $\epsilon_0>0$ depends only on the model-class bounds.
 
The two reduced-form means differ only in the second coordinate by $(\beta-\beta')\lambda Z_1$. Hence the Gaussian Hellinger formula gives
\[
h^2(P_{\beta,\pi,\Omega_0},P_{\beta',\pi,\Omega_0})\lesssim\lambda^2(\beta-\beta')^2=9\lambda^2\epsilon^2 .
\]
Therefore
\[
S_{\mathrm{IV}}(\epsilon)\asymp\lambda^2\epsilon^2 .
\]
 
\paragraph{Target and nuisance complexity.}
Fix $(\pi,\Omega)$ and vary only $\beta$. The conditional means then
differ by
\[
\begin{pmatrix}
0\\
(\beta-\beta')Z^\top\pi
\end{pmatrix},
\]
and the Gaussian Hellinger bound gives
\[
h(P_{\beta,\pi,\Omega},
  P_{\beta',\pi,\Omega})
\lesssim
\|\pi\|\,|\beta-\beta'|
\le
\Lambda|\beta-\beta'|.
\]
Since the covering radius in the definition of
$C_{\mathrm{tar}}(\epsilon)$ is of order
$\sqrt{S_{\mathrm{IV}}(\epsilon)}
\asymp\lambda\epsilon$,
a grid on $[-B,B]$ with spacing of order $\lambda\epsilon$
is sufficient. Hence
\[
C_{\mathrm{tar}}(\epsilon)
\lesssim
\log\frac{1}{\lambda\epsilon}.
\]
 
For nuisance estimation, write
$\eta=(\pi,\Omega)$ and use
\[
d(\eta,\eta')
=
\|\pi-\pi'\|+\|\Omega-\Omega'\|_F.
\]
Let $\widehat\pi$ be the OLS estimator from regressing $X$ on $Z$.
Gaussian regression concentration gives
\[
\|\widehat\pi-\pi\|
\lesssim
\sqrt{\frac{k+\log(1/\delta)}{n}}
\]
with probability at least $1-\delta$, provided
$n\gtrsim k+\log(1/\delta)$.
The reduced-form residual covariance is two-dimensional and is estimated
at rate $O(\sqrt{\log(1/\delta)/n})$. Thus
\[
r_n(\delta)
\lesssim
\sqrt{\frac{k+\log(1/\delta)}{n}}.
\]
 
It remains to verify nuisance stability. For fixed $\beta$, changing
$\pi$ to $\pi'$ changes the conditional mean by
\[
\begin{pmatrix}
Z^\top(\pi-\pi')\\
\beta Z^\top(\pi-\pi')
\end{pmatrix}.
\]
Its squared mean size, averaged over $Z$, is
\[
(1+\beta^2)\|\pi-\pi'\|^2
\le
(1+B^2)\|\pi-\pi'\|^2.
\]
Together with the Gaussian covariance bound on the compact class, this
implies
\[
h(P_{\beta,\pi,\Omega},
  P_{\beta,\pi',\Omega'})
\lesssim
\|\pi-\pi'\|+\|\Omega-\Omega'\|_F
=
d(\eta,\eta').
\]
Hence $L_{\mathrm{nuis}}=O(1)$.
 
Substituting
$S_{\mathrm{IV}}(\epsilon)\asymp\lambda^2\epsilon^2$,
$C_{\mathrm{tar}}(\epsilon)
\lesssim\log(1/(\lambda\epsilon))$,
and the nuisance rate above into Theorem~1 gives
\[
n
\gtrsim
\frac{\log(1/(\lambda\epsilon))+\log(1/\delta)}
{\lambda^2\epsilon^2}
\]
from the identification--target term, while
$L_{\mathrm{nuis}}r_n(\delta)
\lesssim\sqrt{S_{\mathrm{IV}}(\epsilon)}
\asymp\lambda\epsilon$
requires
\[
n
\gtrsim
\frac{k+\log(1/\delta)}
{\lambda^2\epsilon^2}.
\]
This proves the generic bound.
 
Moreover, the pair above satisfies $|\beta-\beta'|=3\epsilon>2\epsilon$, so $S_{\mathrm{IV}}(2\epsilon)\lesssim\lambda^2\epsilon^2$. Decreasing $\epsilon_0$ if necessary, $0<S_{\mathrm{IV}}(2\epsilon)<1/2$ for $\epsilon\le\epsilon_0$, where positivity follows from the lower bound above. Proposition~1 therefore yields
\[
n^*_{\mathrm{PACC}}(\epsilon,\delta)
\gtrsim
\frac{\log(1/\delta)}
{\lambda^2\epsilon^2},
\]
which proves the lower bound in
Theorem~3.
 
\paragraph{Cross-fitted IV upper bound.}
It remains to prove the sharper upper bound for the cross-fitted
estimator. Before the final projection onto $[-B,B]$, write
\[
\widetilde\beta_{\mathrm{CF}}
=
\frac{\sum_{i=1}^n\widehat f_iY_i}
{\sum_{i=1}^n\widehat f_iX_i}.
\]
Using $Y_i=\beta X_i+U_i$,
\begin{equation}
\widetilde\beta_{\mathrm{CF}}-\beta
=
\frac{\sum_{i=1}^n\widehat f_iU_i}
{\sum_{i=1}^n\widehat f_iX_i}.
\label{eq:iv-cf-error}
\end{equation}
Thus the proof reduces to controlling the noise in the numerator and
showing that the denominator retains the first-stage signal.
 
For $i\in I_j$, define
\[
\Delta_j=\widehat\pi^{(-j)}-\pi.
\]
Then
\[
\widehat f_i
=
Z_i^\top\pi+Z_i^\top\Delta_j,
\qquad
X_i=Z_i^\top\pi+V_i.
\]
Since $\widehat\pi^{(-j)}$ is estimated using only observations outside
$I_j$, $\Delta_j$ is independent of
$(Z_i,U_i,V_i)$ for $i\in I_j$. This independence is the main role of
cross-fitting.
 
For a fixed number of folds, Gaussian regression concentration and a
union bound give, simultaneously over $j$,
\[
\|\Delta_j\|
\lesssim
\sqrt{\frac{k+\log(1/\delta)}{n}}
\]
with probability at least $1-\delta/4$. We now condition on the training
samples, so each $\Delta_j$ is fixed on its corresponding held-out fold.
 
For the numerator in \eqref{eq:iv-cf-error},
\[
\sum_i\widehat f_iU_i
=
\sum_i(Z_i^\top\pi)U_i
+
\sum_j\sum_{i\in I_j}(Z_i^\top\Delta_j)U_i.
\]
Because $Z_i$ is independent of $U_i$, the first sum is a sum of
centered products of Gaussian variables. Standard sub-exponential
concentration gives
\[
\left|
\sum_i(Z_i^\top\pi)U_i
\right|
\lesssim
\|\pi\|
\left\{
\sqrt{n\log(1/\delta)}
+
\log(1/\delta)
\right\}.
\]
Conditional on the training samples,
$Z_i^\top\Delta_j$ is likewise independent of $U_i$ and has standard
deviation proportional to $\|\Delta_j\|$. Applying the same
concentration bound within each fold and using the bound on
$\|\Delta_j\|$ yields
\[
\left|
\sum_j\sum_{i\in I_j}
(Z_i^\top\Delta_j)U_i
\right|
\lesssim
\sqrt{k\log(1/\delta)}
+
\log(1/\delta).
\]
Therefore
\[
\left|
\sum_i\widehat f_iU_i
\right|
\lesssim
\|\pi\|\sqrt{n\log(1/\delta)}
+
\sqrt{k\log(1/\delta)}
+
\log(1/\delta),
\]
where fixed model-class constants are absorbed into $\lesssim$.
 
For the denominator,
\[
\sum_i\widehat f_iX_i
=
\sum_i(Z_i^\top\pi)^2
+\sum_i(Z_i^\top\pi)V_i
+\sum_j\sum_{i\in I_j}(Z_i^\top\Delta_j)(Z_i^\top\pi)
+\sum_j\sum_{i\in I_j}(Z_i^\top\Delta_j)V_i.
\]
The first term is the signal. Since
$Z_i^\top\pi\sim N(0,\|\pi\|^2)$,
chi-square concentration gives
\[
\sum_i(Z_i^\top\pi)^2
=
n\|\pi\|^2
+
O\!\left(
\|\pi\|^2
\{\sqrt{n\log(1/\delta)}+\log(1/\delta)\}
\right).
\]
The second term is centered and satisfies
\[
\left|
\sum_i(Z_i^\top\pi)V_i
\right|
\lesssim
\|\pi\|
\{\sqrt{n\log(1/\delta)}+\log(1/\delta)\}.
\]
 
For the third term, conditional on the training data,
\[
\mathbb E[
(Z_i^\top\Delta_j)(Z_i^\top\pi)
\mid\Delta_j]
=
\Delta_j^\top\pi.
\]
The Gaussian OLS representation also gives
\[
|\Delta_j^\top\pi|
\lesssim
\|\pi\|
\sqrt{\frac{\log(1/\delta)}{n}}
\]
simultaneously over the fixed number of folds. Its centered fluctuation,
as well as the final term involving $V_i$, is bounded by
\[
\lesssim
\|\pi\|\sqrt{k\log(1/\delta)}
+
\sqrt{k\log(1/\delta)}
+
\log(1/\delta).
\]
Combining these bounds, the denominator is a constant fraction of
$n\|\pi\|^2$ whenever
\[
n
\gtrsim
k
+
\frac{\sqrt{k\log(1/\delta)}}{\lambda^2}
+
\frac{\log(1/\delta)}{\lambda^2}.
\]
In particular,
\[
\sum_i\widehat f_iX_i
\gtrsim
n\|\pi\|^2.
\]
 
Substituting the numerator and denominator bounds into
\eqref{eq:iv-cf-error} and using $\|\pi\|\ge\lambda$ gives
\[
|\widetilde\beta_{\mathrm{CF}}-\beta|
\lesssim
\sqrt{\frac{\log(1/\delta)}
{n\lambda^2}}
+
\frac{
\sqrt{k\log(1/\delta)}+\log(1/\delta)
}{
n\lambda^2
}.
\]
For $0<\epsilon\le1$, this is at most $\epsilon$ whenever
\[
n
\gtrsim
k
+
\frac{\sqrt{k\log(1/\delta)}}
{\lambda^2\epsilon}
+
\frac{\log(1/\delta)}
{\lambda^2\epsilon^2}.
\]
Finally, since the true effect satisfies $\beta\in[-B,B]$, projection
onto $[-B,B]$ cannot increase the estimation error. This proves the
upper bound in Theorem~3.

\section{Proofs for Propensity-Score-Based CATE Estimation}
\label{app:ps}
 
Throughout, $\|g\|_2:=\|g\|_{L_2(P_X)}$, covering numbers of $\mathcal F$ are
proper (centres in $\mathcal F$), and $P_{\tau,e,m}:=P_{\tau,\eta}$ for
$\eta=(e,m)\in\mathcal H$. Under $P_{\tau,e,m}$,
$\mathbb E[Y\mid X,A=a]=m(X)+\{a-e(X)\}\tau(X)$, $m=\mathbb E[Y\mid X]$ and $\tau$ is the CATE. Every
$(\tau,\eta)\in\mathcal F\times\mathcal H$ indexes a model, so the
parametrization is variation independent as Theorem~1 requires.
 
All Hellinger calculations below compare two models $P_{\tau,e,m}$ and
$P_{\tau',e',m'}$ through the difference of their conditional outcome means in treatment arm $a$,
\[
  D_a=(m-m')+(a-e)\tau-(a-e')\tau',\qquad a\in\{0,1\}.
\]
Two facts are used repeatedly: $D_1-D_0=\tau-\tau'$, so the CATE difference
is the contrast between the two treatment arms; and, since $|a-e|\le1$,
$|D_a|\le|m-m'|+|\tau|+|\tau'|\le4B$.
 
\subsection{Two lemmas}
 
\begin{lemma}[R-loss; \citealp{robinson1988root,nie2021quasi}]\label{lem:F-rloss}
Assume consistency, unconfoundedness, overlap and $\mathbb E[Y^2]<\infty$,
and let $U=Y-\mathbb E[Y\mid X,A]$. Then
$Y-m_0(X)=\{A-e_0(X)\}\tau_0(X)+U$ with $\mathbb E[U\mid X,A]=0$, and
$R(f)-R(\tau_0)=\mathbb E[e_0(1-e_0)(f-\tau_0)^2]$.
\end{lemma}
 
\begin{proof}
Under consistency and unconfoundedness,
$\mathbb E[Y\mid X,A]=\mathbb E[Y\mid X,A=0]+A\tau_0(X)$. Averaging over
$A\mid X$ gives $m_0=\mathbb E[Y\mid X,A=0]+e_0\tau_0$, and subtracting the
two identities gives $\mathbb E[Y\mid X,A]-m_0=(A-e_0)\tau_0$, which is the
first claim. Consequently
\[
  Y-m_0-(A-e_0)f=(A-e_0)(\tau_0-f)+U .
\]
Squaring and taking expectations, the cross term vanishes because
$\mathbb E[U\mid X,A]=0$, and $\mathbb E[(A-e_0)^2\mid X]=e_0(1-e_0)$. Hence
$R(f)=\mathbb E[e_0(1-e_0)(f-\tau_0)^2]+\mathbb E[U^2]$; setting $f=\tau_0$
gives $R(\tau_0)=\mathbb E[U^2]$.
\end{proof}
 
\begin{lemma}[Hellinger decomposition]\label{lem:F-decomp}
With all functions evaluated at $X\sim P_X$,
\begin{align*}
  h^2(P_{\tau,e,m},P_{\tau',e',m'})
  ={}&\mathbb E\big[1-\sqrt{ee'}-\sqrt{(1-e)(1-e')}\big]
  +\mathbb E\Big[\sqrt{e e'}\,\Big\{1-\exp\Big(-\frac{D_1^2}{8\sigma^2}\Big)\Big\}\Big]\\
  &+\mathbb E\Big[\sqrt{(1-e)(1-e')}\,\Big\{1-\exp\Big(-\frac{D_0^2}{8\sigma^2}\Big)\Big\}\Big].
\end{align*}
The first term, which we call the treatment term, is the squared Hellinger
distance between the two propensity laws; the other two, the outcome terms,
measure the separation of the outcome means within each arm. All three are
nonnegative.
\end{lemma}
 
\begin{proof}
Both laws share the covariate distribution $P_X$, so their densities differ only through
$A\mid X$ and $Y\mid X,A$. Writing $\varphi_a$ and $\varphi'_a$ for the two
normal densities of $Y$ given $(X,A=a)$, the Hellinger affinity factorizes as
\[
  \int\sqrt{dP_{\tau,e,m}\,dP_{\tau',e',m'}}
  =\mathbb E\Big[\sqrt{ee'}\int\!\sqrt{\varphi_1\varphi'_1}
   +\sqrt{(1-e)(1-e')}\int\!\sqrt{\varphi_0\varphi'_0}\Big].
\]
For two normal densities with common variance $\sigma^2$ and means differing
by $D$, completing the square gives $\int\sqrt{\varphi\varphi'}=\exp\{-D^2/(8\sigma^2)\}$.
Hence the affinity equals
\[
\mathbb E\big[\sqrt{ee'}\exp\{-D_1^2/(8\sigma^2)\}+\sqrt{(1-e)(1-e')}\exp\{-D_0^2/(8\sigma^2)\}\big].
\]
Subtract it from $1$, and write each $\exp(\cdot)$ as $1-\{1-\exp(\cdot)\}$:
the terms without braces combine into
$1-\sqrt{ee'}-\sqrt{(1-e)(1-e')}$, which is the treatment term, and the
remaining terms are the outcome terms.
\end{proof}
 
All bounds below combine Lemma~\ref{lem:F-decomp} with the elementary
inequality $t/(1+t)\le1-\exp(-t)\le t$ for $t\ge0$, applied with
$t=D_a^2/(8\sigma^2)$.
 
\subsection{Proof of Proposition 2}
\label{app:ps-prop}
 
\paragraph{(i), lower bound.}
Identification strength is a lower bound on $h^2$ over $\epsilon$-separated
pairs, so it suffices to bound $h^2$ below by $\|\tau-\tau'\|_2^2$. The
treatment term is nonnegative, so we drop it and keep only the outcome
terms. Since $|D_a|\le4B$, we have $t=D_a^2/(8\sigma^2)\le2B^2/\sigma^2$, and
therefore
\[
  1-\exp\Big(-\frac{D_a^2}{8\sigma^2}\Big)\ge\frac{t}{1+2B^2/\sigma^2}
  =\frac{D_a^2}{8(\sigma^2+2B^2)} .
\]
By overlap, both weights $\sqrt{ee'}$ and $\sqrt{(1-e)(1-e')}$ are at least
$\alpha$. Since $D_1-D_0=\tau-\tau'$, at least one arm carries half of the CATE
difference: $D_1^2+D_0^2\ge\tfrac12(\tau-\tau')^2$. Combining,
\[
  h^2(P_{\tau,e,m},P_{\tau',e',m'})
  \ge\frac{\alpha\,\mathbb E[D_1^2+D_0^2]}{8(\sigma^2+2B^2)}
  \ge\frac{\alpha(1-\alpha)}{32(\sigma^2+B^2)}\|\tau-\tau'\|_2^2 ,
\]
using $\alpha\ge\alpha(1-\alpha)$ and $\sigma^2+2B^2\le2(\sigma^2+B^2)$.
Every pair in the infimum defining $S_{\mathcal M_{\mathrm{PS}}}(\epsilon)$
has $\|\tau-\tau'\|_2>\epsilon$; if there is no such pair,
$S_{\mathcal M_{\mathrm{PS}}}(\epsilon)=1$. In both cases, for $0<\epsilon\le2B$,
\begin{equation}\label{eq:F-S-lower}
  S_{\mathcal M_{\mathrm{PS}}}(\epsilon)\ge s_\epsilon:=\frac{\alpha(1-\alpha)\epsilon^2}{32(\sigma^2+B^2)} .
\end{equation}
 
\paragraph{(i), upper bound.}
To show the overlap factor cannot be improved, we exhibit a pair that is
causally separated but observationally close. Take two models with the same
nuisances, $e=e'\equiv\alpha$ (which lies in $\mathcal H$) and $m=m'$, so
that only the CATE differs. The treatment term vanishes, and
$D_1=(1-\alpha)(\tau-\tau')$, $D_0=-\alpha(\tau-\tau')$. Using
$1-\exp(-t)\le t$ in both outcome terms,
\begin{equation}\label{eq:F-S-upper}
  h^2(P_{\tau,\alpha,m},P_{\tau',\alpha,m})
  \le\frac{\alpha(1-\alpha)^2+(1-\alpha)\alpha^2}{8\sigma^2}\|\tau-\tau'\|_2^2
  =\frac{\alpha(1-\alpha)}{8\sigma^2}\|\tau-\tau'\|_2^2 .
\end{equation}
If $\epsilon<\|\tau-\tau'\|_2\le2\epsilon$, this pair enters the infimum
defining $S_{\mathcal M_{\mathrm{PS}}}(\epsilon)$, so
$S_{\mathcal M_{\mathrm{PS}}}(\epsilon)\le\alpha(1-\alpha)\epsilon^2/(2\sigma^2)$.
 
\paragraph{(ii).}
By Definition~3, $C_{\mathrm{tar}}(\epsilon)$ is the log covering number of
$\mathcal P_\eta=\{P_{\tau,\eta}:\tau\in\mathcal F\}$ at Hellinger radius
$c_0\sqrt{S_{\mathcal M_{\mathrm{PS}}}(\epsilon)}$. We bound it by an
$L_2(P_X)$ covering number of $\mathcal F$. Fix $\eta=(e,m)$ and compare
$P_{\tau,e,m}$ with $P_{\tau',e,m}$. The treatment term vanishes,
$D_1=(1-e)(\tau-\tau')$ and $D_0=-e(\tau-\tau')$, so, as in
\eqref{eq:F-S-upper},
\[
  h^2(P_{\tau,e,m},P_{\tau',e,m})\le\frac{\mathbb E[e(1-e)(\tau-\tau')^2]}{8\sigma^2}
  \le\frac{\|\tau-\tau'\|_2^2}{32\sigma^2}.
\]
Thus $h\le\|\tau-\tau'\|_2/(4\sqrt2\,\sigma)$: if $\{\tau_k\}$ is an
$L_2(P_X)$ cover of $\mathcal F$ at radius $4\sqrt2\,\sigma r$, then
$\{P_{\tau_k,\eta}\}$ is a Hellinger cover of $\mathcal P_\eta$ at radius $r$.
Covering numbers decrease as the radius grows, and
$S_{\mathcal M_{\mathrm{PS}}}(\epsilon)\ge s_\epsilon$ by \eqref{eq:F-S-lower},
so
\[
  N\big(c_0\sqrt{S_{\mathcal M_{\mathrm{PS}}}(\epsilon)},\mathcal P_\eta,h\big)
  \le N\big(c_0\sqrt{s_\epsilon},\mathcal P_\eta,h\big)
  \le N\big(4\sqrt2\,\sigma c_0\sqrt{s_\epsilon},\mathcal F,L_2(P_X)\big),
\]
and $4\sqrt2\,\sigma\sqrt{s_\epsilon}=\sigma\sqrt{\alpha(1-\alpha)}\,\epsilon/\sqrt{\sigma^2+B^2}$.
The bound does not depend on $\eta$, which proves (ii).
 
\paragraph{(iii).}
Now fix $\tau$ and vary only the nuisances. Then
$D_1=D_0=(m-m')-(e-e')\tau$, so
$\|D_1\|_2\le\|m-m'\|_2+B\|e-e'\|_2$. For the treatment term, write
\[
1-\sqrt{ee'}-\sqrt{(1-e)(1-e')}=\tfrac12(\sqrt e-\sqrt{e'})^2+\tfrac12(\sqrt{1-e}-\sqrt{1-e'})^2
\]
and use
\[
(\sqrt e-\sqrt{e'})^2=\frac{(e-e')^2}{(\sqrt e+\sqrt{e'})^2}\le\frac{(e-e')^2}{4\alpha},
\]
and likewise for $1-e,1-e'\ge\alpha$; this gives at most
$\|e-e'\|_2^2/(4\alpha)$. For the outcome terms, $1-\exp(-t)\le t$ and
$\sqrt{ee'}+\sqrt{(1-e)(1-e')}\le1$ (Cauchy--Schwarz) give at most
$\|D_1\|_2^2/(8\sigma^2)$. Adding and taking square roots,
\[
  h(P_{\tau,e,m},P_{\tau,e',m'})\le\frac{\|e-e'\|_2}{2\sqrt\alpha}
  +\frac{\|m-m'\|_2+B\|e-e'\|_2}{2\sqrt2\,\sigma},
\]
which is (iii). In the notation of Section~2.2, with
$d(\eta,\eta')=\|e-e'\|_2+\|m-m'\|_2$, the stability condition holds with
$L_{\mathrm{nuis}}=\alpha^{-1/2}+(1+B)/\sigma$.
 
\begin{corollary}[Generic bound]\label{cor:F-generic}
Let $0<\epsilon\le2B$ and $0<\delta\le1/8$, assume the conditions of
Appendix~C, and let $(\hat e,\hat m)$ take values in $\mathcal H$ (e.g.\
after clipping) with rate $r_n(\delta)$ in $d$. Then Theorem~1 gives an
$(\epsilon,\delta)$-PACC guarantee whenever
\[
  n\gtrsim\frac{\sigma^2+B^2}{\alpha(1-\alpha)\epsilon^2}\big\{C_{\mathrm{tar}}(\epsilon)+\log(1/\delta)\big\}
\]
and
\[
  r_n(\delta)\le c\,\frac{\alpha\epsilon}{\sqrt{\sigma^2+B^2}\,\{1+(1+B)/\sigma\}}
\]
for a universal constant $c>0$.
\end{corollary}
 
\begin{proof}
Condition (1) of Theorem~1 follows from the sample-size condition and
$S_{\mathcal M_{\mathrm{PS}}}(\epsilon)\ge s_\epsilon$. Condition (2) requires
$L_{\mathrm{nuis}}r_n(\delta)\le c_*\sqrt{s_\epsilon}$. Since $\alpha\le1$,
$L_{\mathrm{nuis}}\le\alpha^{-1/2}\{1+(1+B)/\sigma\}$, and since
$1-\alpha\ge1/2$, $\sqrt{s_\epsilon}\ge\sqrt\alpha\,\epsilon/(8\sqrt{\sigma^2+B^2})$.
Dividing, condition (2) holds when $r_n(\delta)$ satisfies the display with
$c=c_*/8$.
\end{proof}
 
\subsection{Proof of Theorem 4}
\label{app:ps-orth}
 
Let $\hat R_n(f)$ denote the cross-fitted objective of Section~5, divided by
$n$, and let $c=1/256$.
 
\begin{assumption}[Conditions for Theorem 4]\label{asm:F-orth}
\leavevmode
\begin{enumerate}
\item[(A1)] The model class consists of all laws of $(X,A,Y)$ satisfying
  consistency, unconfoundedness, $(e_0,m_0)\in\mathcal H$,
  $\tau_0\in\mathcal F$, and
  $\mathbb E[\exp(sU)\mid X,A]\le\exp(s^2\sigma^2/2)$ a.s.\ for all
  $s\in\mathbb R$, with $U$ as in Lemma~\ref{lem:F-rloss}. It contains
  $\mathcal M_{\mathrm{PS}}$.
\item[(A2)] $|I_j|\ge\lfloor n/J\rfloor$, and $\hat e^{(-j)},\hat m^{(-j)}$
  are measurable functions of the data outside $I_j$, with
  $\hat e^{(-j)}\in[0,1]$ and $\|\hat m^{(-j)}\|_\infty\le B$.
\item[(A3)] Uniformly over the class, the event
  $\mathcal E_{\mathrm{nuis}}=\{\max_j\|b_j\|_2\le c\,\alpha(1-\alpha)\epsilon\}$
  has probability at least $1-\delta/2$.
\item[(A4)] $\hat\tau\in\mathcal F$ and
  $\hat R_n(\hat\tau)\le\inf_{f\in\mathcal F}\hat R_n(f)+\alpha(1-\alpha)\epsilon^2/64$.
\end{enumerate}
\end{assumption}

Because $\hat\tau$ nearly minimizes $\hat R_n$ and $\tau_0\in\mathcal F$, it
suffices to show that with probability at least $1-\delta$, every
$f\in\mathcal F$ with $\|f-\tau_0\|_2>\epsilon$ has
\[
\hat R_n(f)-\hat R_n(\tau_0)>\alpha(1-\alpha)\epsilon^2/64 .
\]
Since $\epsilon\le2B$ and $\alpha(1-\alpha)\le1/4$, the prefactor
$(\sigma^2+B^2)/\{\alpha(1-\alpha)\epsilon^2\}$ in Theorem~4 is at least $1$;
hence its sample-size condition also gives $n\ge2J$, so
$n_j:=|I_j|\ge n/(2J)$, and $n\ge\log(4/\delta)$.
 
Fix a fold $j$. Conditional
on the data outside $I_j$, denoted $\mathcal D_{-j}$, the observations in
$I_j$ are i.i.d.\ and the nuisance estimates are fixed functions. Suppressing
the superscript $(-j)$, write $\Delta_e=\hat e-e_0$, $\Delta_m=\hat m-m_0$,
$\tilde A=A-\hat e(X)$ and $\psi=f-\tau_0$, and let an overbar denote the
average over $i\in I_j$.
 
\paragraph{Step 1: orthogonality.}
By Lemma~\ref{lem:F-rloss}, $Y-m_0=(A-e_0)\tau_0+U$, and
$A-e_0=\tilde A+\Delta_e$. Hence the residual of $f$ in the R-loss is
\[
  Y-\hat m-\tilde Af
  =\underbrace{(\Delta_e\tau_0-\Delta_m+U)}_{\text{residual of }\tau_0}-\tilde A\psi .
\]
Since $(x-\tilde A\psi)^2-x^2=\tilde A^2\psi^2-2\tilde A\psi x$, the loss
difference between $f$ and $\tau_0$ at one observation is
\[
  \ell=\underbrace{\tilde A^2\psi^2}_{T_1}
   +\underbrace{2\tilde A(\Delta_m-\tau_0\Delta_e)\psi}_{T_2}
   \underbrace{-\,2\tilde AU\psi}_{T_3},
\]
and $\hat R_n(f)-\hat R_n(\tau_0)=\sum_j(n_j/n)\,\bar\ell$. We now take
conditional expectations. First,
\[
\mathbb E[\tilde A^2\mid X]=\mathrm{Var}(A\mid X)+\Delta_e^2=e_0(1-e_0)+\Delta_e^2,
\]
so $\mathbb E[T_1\mid\mathcal D_{-j}]=:\Gamma(\psi)\ge\alpha(1-\alpha)\|\psi\|_2^2$
is the curvature. Second, $\mathbb E[\tilde A\mid X]=-\Delta_e$, so
\[
  \mathbb E[T_2\mid\mathcal D_{-j}]=-2\,\mathbb E[\Delta_e(\Delta_m-\tau_0\Delta_e)\psi]
  =-2\,\mathbb E[b_j\psi].
\]
This is the orthogonality step: $T_2$ is linear in $\Delta_m$ and
$\Delta_e$, but its mean picks up a second factor $\Delta_e$, so only the
product $b_j$ survives. Third, $\mathbb E[T_3\mid\mathcal D_{-j}]=0$ because
$\mathbb E[U\mid X,A]=0$. On $\mathcal E_{\mathrm{nuis}}$ and for
$\|\psi\|_2\ge\epsilon/2$, Cauchy--Schwarz gives
\[
2|\mathbb E[b_j\psi]|\le2c\,\alpha(1-\alpha)\epsilon\|\psi\|_2\le4c\,\alpha(1-\alpha)\|\psi\|_2^2\le\Gamma(\psi)/4,
\]
so the bias costs at most a quarter of the curvature:
\begin{equation}\label{eq:F-curv}
  \mathbb E[\ell\mid\mathcal D_{-j}]\ge\tfrac34\Gamma(\psi)
  \qquad\text{whenever }\|\psi\|_2\ge\epsilon/2 .
\end{equation}
 
\paragraph{Step 2: a fixed candidate.}
Fix $g\in\mathcal F$ with $\|g-\tau_0\|_2\ge\epsilon/2$, let $\psi=g-\tau_0$
and $\Gamma=\Gamma(\psi)$, and work on $\mathcal E_{\mathrm{nuis}}$. We show
that $\bar\ell$ stays within $\Gamma/2$ of its mean. Because $|\tilde A|\le1$,
$|\psi|\le2B$ and $|\Delta_m-\tau_0\Delta_e|\le3B$, we have $0\le T_1\le4B^2$
and $|T_2|\le12B^2$; moreover $T_1^2\le4B^2T_1$ and $T_2^2\le36B^2T_1$, so
both variances are at most a constant times $B^2\Gamma$. Bernstein's
inequality, whose deviation term scales with this variance, then gives
\[
  \overline{T_1+T_2}\ge\mathbb E[T_1+T_2\mid\mathcal D_{-j}]-\Gamma/4
  \quad\text{and}\quad\overline{T_1}\le2\Gamma,
\]
each except with probability $\exp(-c'n_j\Gamma/B^2)$. For the noise term,
condition also on $\{(X_i,A_i)\}_{i\in I_j}$: the $T_{3,i}$ are then
independent and sub-Gaussian, with total variance proxy
\[
4\sigma^2\sum_{i\in I_j}\tilde A_i^2\psi(X_i)^2=4\sigma^2n_j\overline{T_1}.
\]
On
$\{\overline{T_1}\le2\Gamma\}$ this gives $\overline{T_3}\ge-\Gamma/4$ except
with probability $\exp(-c'n_j\Gamma/\sigma^2)$. Outside these three events,
\eqref{eq:F-curv} yields
\[
  \bar\ell\ge\tfrac34\Gamma-\tfrac14\Gamma-\tfrac14\Gamma=\tfrac14\Gamma
  \ge\tfrac1{16}\alpha(1-\alpha)\epsilon^2 .
\]
Using $\Gamma\ge\alpha(1-\alpha)\epsilon^2/4$ and $n_j\ge n/(2J)$, the failure
probability for fold $j$ is at most
\[
3\exp\Big\{-\frac{c''n\,\alpha(1-\alpha)\epsilon^2}{J(\sigma^2+B^2)}\Big\}.
\]
This holds for
every $\mathcal D_{-j}$ in $\mathcal E_{\mathrm{nuis}}$, hence also
unconditionally on $\mathcal E_{\mathrm{nuis}}$.
 
\paragraph{Step 3: union bound over a net.}
Let $r=c\,\alpha(1-\alpha)\epsilon^2/(\sigma+B)$ and let $\mathcal N_r$ be a
proper $\|\cdot\|_\infty$-net of $\mathcal F$ of size
$N(r,\mathcal F,\|\cdot\|_\infty)$. Applying Step~2 to every $g\in\mathcal N_r$
with $\|g-\tau_0\|_2\ge\epsilon/2$ and every fold, the total failure
probability on $\mathcal E_{\mathrm{nuis}}$ is at most
\[
  3J\,N(r,\mathcal F,\|\cdot\|_\infty)\exp\Big\{-\frac{c''n\,\alpha(1-\alpha)\epsilon^2}{J(\sigma^2+B^2)}\Big\}.
\]
This is at most $\delta/4$ exactly when
\[
n\gtrsim_J\frac{\sigma^2+B^2}{\alpha(1-\alpha)\epsilon^2}\Big\{\log N(r,\mathcal F,\|\cdot\|_\infty)+\log(1/\delta)\Big\},
\]
which is the sample-size condition of Theorem~4 (the additive $\log(12J)$ is
absorbed into $\log(1/\delta)$ since $\delta\le1/8$). Averaging over folds,
outside this event,
\begin{equation}\label{eq:F-gap}
  \hat R_n(g)-\hat R_n(\tau_0)\ge\tfrac1{16}\alpha(1-\alpha)\epsilon^2
\end{equation}
for every such $g$.
 
\paragraph{Step 4: from the net to $\mathcal F$.}
For $f\in\mathcal F$ and $g\in\mathcal N_r$ with $\|f-g\|_\infty\le r$, we
bound $|\hat R_n(f)-\hat R_n(g)|$. At observation $i$, the residuals of $f$
and $g$ differ by $\tilde A_i(g-f)(X_i)$, which is at most $r$, and their sum
is at most $2|Y_i-\hat m(X_i)|+2B$ in absolute value. Since
$Y_i-\hat m(X_i)=(m_0-\hat m)+(A_i-e_0)\tau_0+U_i$ has absolute value at most
$3B+|U_i|$, the squared residuals differ by at most $r(8B+2|U_i|)$. To control
the noise,
\[
\mathbb E[\exp(s|U|)]\le\mathbb E[\exp(sU)+\exp(-sU)]\le2\exp(s^2\sigma^2/2),
\]
and Chernoff's bound with $s=2/\sigma$ gives $n^{-1}\sum_i|U_i|\le2\sigma$
except with probability $(2e^{-2})^n\le e^{-n}\le\delta/4$. On this event,
\[
  |\hat R_n(f)-\hat R_n(g)|\le r(8B+4\sigma)\le8r(\sigma+B)
  =\tfrac1{32}\alpha(1-\alpha)\epsilon^2 .
\]
 
\paragraph{Step 5: conclusion.}
The exceptional events of (A3), Step~3 and Step~4 have total probability at
most $\delta/2+\delta/4+\delta/4=\delta$. Outside them, let
$f\in\mathcal F$ with $\|f-\tau_0\|_2>\epsilon$ and pick $g\in\mathcal N_r$
with $\|f-g\|_\infty\le r$. Since $\epsilon\le2B$ gives $r\le\epsilon/512$,
$\|g-\tau_0\|_2>\epsilon/2$, so \eqref{eq:F-gap} applies to $g$, and Step~4
transfers it to $f$:
\[
  \hat R_n(f)-\hat R_n(\tau_0)
  \ge\tfrac1{16}\alpha(1-\alpha)\epsilon^2-\tfrac1{32}\alpha(1-\alpha)\epsilon^2
  =\tfrac1{32}\alpha(1-\alpha)\epsilon^2 .
\]
This exceeds the optimization tolerance in (A4), so $\hat\tau$ is not such an
$f$, i.e.\ $\|\hat\tau-\tau_0\|_2\le\epsilon$. \qed
 
\subsection{Overlap lower bound}
\label{app:ps-lower}
 
\begin{lemma}\label{lem:F-lower}
Suppose $\mathcal F$ contains $\tau_\pm$ with
$2\epsilon<\|\tau_+-\tau_-\|_2\le4\epsilon$, where
$0<\epsilon<\min(\sigma,B)$. Then, for $0<\delta\le1/8$, both
$\mathcal M_{\mathrm{PS}}$ and the class (A1) satisfy
$n^*_{\mathrm{PACC}}(\epsilon,\delta)\gtrsim\sigma^2\log(1/\delta)/\{\alpha(1-\alpha)\epsilon^2\}$.
\end{lemma}
 
\begin{proof}
We apply Proposition~1, which requires
$0<S_{\mathcal M_{\mathrm{PS}}}(2\epsilon)<1/2$. The pair
$P_{\tau_\pm,\alpha,0}$ is $2\epsilon$-separated, and by
\eqref{eq:F-S-upper} its squared Hellinger distance is at most
\[
\frac{\alpha(1-\alpha)(4\epsilon)^2}{8\sigma^2}=\frac{2\alpha(1-\alpha)\epsilon^2}{\sigma^2},
\]
which is below $1/2$ because $\alpha(1-\alpha)\le1/4$ and $\epsilon<\sigma$.
Hence $S_{\mathcal M_{\mathrm{PS}}}(2\epsilon)\le2\alpha(1-\alpha)\epsilon^2/\sigma^2<1/2$,
and $S_{\mathcal M_{\mathrm{PS}}}(2\epsilon)\ge s_{2\epsilon}>0$ by
\eqref{eq:F-S-lower} since $2\epsilon<2B$. Proposition~1 then gives
\[
n^*_{\mathrm{PACC}}(\epsilon,\delta)\gtrsim\frac{\log(1/\delta)}{S_{\mathcal M_{\mathrm{PS}}}(2\epsilon)},
\]
which is the claim for $\mathcal M_{\mathrm{PS}}$. The model class of
Assumption~\ref{asm:F-orth}~(A1) contains $\mathcal M_{\mathrm{PS}}$, and
enlarging a class cannot decrease $n^*_{\mathrm{PACC}}$.
\end{proof}
 
If instead $\alpha$ is not bounded away from zero over the class, the bound
\eqref{eq:F-S-upper} tends to zero as $\alpha\downarrow0$ for a fixed pair
$\tau,\tau'$ with $\|\tau-\tau'\|_2>2\epsilon$. The identification strength
at scale $2\epsilon$ is then zero, and Proposition~1 rules out any finite
sample size, even though every model has positive overlap and hence an
identified CATE (Lemma~\ref{lem:F-rloss}).

%% file: refs.bib
@book{pearl2009causality,
  author    = {Pearl, Judea},
  title     = {Causality: Models, Reasoning, and Inference},
  edition   = {2},
  publisher = {Cambridge University Press},
  address   = {Cambridge},
  year      = {2009}
}

@book{imbens2015causal,
  author    = {Imbens, Guido W. and Rubin, Donald B.},
  title     = {Causal Inference for Statistics, Social, and Biomedical Sciences: An Introduction},
  publisher = {Cambridge University Press},
  address   = {Cambridge},
  year      = {2015},
  doi       = {10.1017/CBO9781139025751}
}

@book{hernan2020causal,
  author    = {Hern{\'a}n, Miguel A. and Robins, James M.},
  title     = {Causal Inference: What If},
  publisher = {Chapman \& Hall/CRC},
  address   = {Boca Raton},
  year      = {2020}
}

@book{angrist2009mostly,
  title={Mostly Harmless Econometrics: An Empiricist's Companion},
  author={Angrist, Joshua D. and Pischke, J{\"o}rn-Steffen},
  year={2009},
  publisher={Princeton University Press}
}

@article{angrist1996identification,
  title={Identification of Causal Effects Using Instrumental Variables},
  author={Angrist, Joshua D. and Imbens, Guido W. and Rubin, Donald B.},
  journal={Journal of the American Statistical Association},
  volume={91},
  number={434},
  pages={444--455},
  year={1996}
}

@article{farrington1995relative,
  author  = {Farrington, C. Paddy},
  title   = {Relative Incidence Estimation from Case Series for Vaccine Safety Evaluation},
  journal = {Biometrics},
  volume  = {51},
  number  = {1},
  pages   = {228--235},
  year    = {1995},
  doi     = {10.2307/2533328}
}

@article{whitaker2006tutorial,
  author  = {Whitaker, Heather J. and Farrington, C. Paddy and Spiessens, Bart and Musonda, Patrick},
  title   = {Tutorial in Biostatistics: The Self-Controlled Case Series Method},
  journal = {Statistics in Medicine},
  volume  = {25},
  number  = {10},
  pages   = {1768--1797},
  year    = {2006},
  doi     = {10.1002/sim.2302}
}

@article{musonda2006sample,
  author  = {Musonda, Patrick and Farrington, C. Paddy and Whitaker, Heather J.},
  title   = {Sample Sizes for Self-Controlled Case Series Studies},
  journal = {Statistics in Medicine},
  volume  = {25},
  number  = {15},
  pages   = {2618--2631},
  year    = {2006},
  doi     = {10.1002/sim.2477}
}

@article{petersen2016self,
  author  = {Petersen, Irene and Douglas, Ian and Whitaker, Heather},
  title   = {Self Controlled Case Series Methods: An Alternative to Standard Epidemiological Study Designs},
  journal = {BMJ},
  volume  = {354},
  pages   = {i4515},
  year    = {2016},
  doi     = {10.1136/bmj.i4515}
}

@article{etievant2026causal,
  author  = {{\'E}ti{\'e}vant, Lola and Gail, Mitchell H. and Follmann, Dean},
  title   = {A Causal Framework for the Self-Controlled Case Series Design},
  journal = {Journal of Causal Inference},
  volume  = {14},
  number  = {1},
  pages   = {20240074},
  year    = {2026},
  doi     = {10.1515/jci-2024-0074}
}

@article{davies2015many,
  author  = {Davies, Neil M. and von Hinke Kessler Scholder, Stephanie and Farbmacher, Helmut and Burgess, Stephen and Windmeijer, Frank and Davey Smith, George},
  title   = {The Many Weak Instruments Problem and {M}endelian Randomization},
  journal = {Statistics in Medicine},
  volume  = {34},
  number  = {3},
  pages   = {454--468},
  year    = {2015}
}

@article{dobbie2018effects,
  author  = {Dobbie, Will and Goldin, Jacob and Yang, Crystal S.},
  title   = {The Effects of Pretrial Detention on Conviction, Future Crime, and Employment: Evidence from Randomly Assigned Judges},
  journal = {American Economic Review},
  volume  = {108},
  number  = {2},
  pages   = {201--240},
  year    = {2018}
}

@article{bekker1994,
  title={Alternative Approximations to the Distributions of Instrumental Variable Estimators},
  author={Bekker, Paul A.},
  journal={Econometrica},
  volume={62},
  number={3},
  pages={657--681},
  year={1994}
}

@article{angrist1999jive,
  title={Jackknife Instrumental Variables Estimation},
  author={Angrist, Joshua D. and Imbens, Guido W. and Krueger, Alan B.},
  journal={Journal of Applied Econometrics},
  volume={14},
  number={1},
  pages={57--67},
  year={1999}
}

@article{donald2001choosing,
  author  = {Donald, Stephen G. and Newey, Whitney K.},
  title   = {Choosing the Number of Instruments},
  journal = {Econometrica},
  volume  = {69},
  number  = {5},
  pages   = {1161--1191},
  year    = {2001}
}

@misc{staiger1994instrumental,
  title={Instrumental variables regression with weak instruments},
  author={Staiger, Douglas O and Stock, James H},
  year={1994},
  publisher={National Bureau of Economic Research Cambridge, Mass., USA}
}

@article{chao2005consistent,
  author  = {Chao, John C. and Swanson, Norman R.},
  title   = {Consistent Estimation with a Large Number of Weak Instruments},
  journal = {Econometrica},
  volume  = {73},
  number  = {5},
  pages   = {1673--1692},
  year    = {2005}
}

@article{robinson1988root,
  author  = {Robinson, P. M.},
  title   = {Root-{N}-Consistent Semiparametric Regression},
  journal = {Econometrica},
  volume  = {56},
  number  = {4},
  pages   = {931--954},
  year    = {1988},
  doi     = {10.2307/1912705}
}

@article{nie2021quasi,
  author  = {Nie, Xinkun and Wager, Stefan},
  title   = {Quasi-Oracle Estimation of Heterogeneous Treatment Effects},
  journal = {Biometrika},
  volume  = {108},
  number  = {2},
  pages   = {299--319},
  year    = {2021},
  doi     = {10.1093/biomet/asaa076}
}

@article{chernozhukov2018double,
  author  = {Chernozhukov, Victor and Chetverikov, Denis and Demirer, Mert and Duflo, Esther and Hansen, Christian and Newey, Whitney and Robins, James},
  title   = {Double/Debiased Machine Learning for Treatment and Structural Parameters},
  journal = {The Econometrics Journal},
  volume  = {21},
  number  = {1},
  pages   = {C1--C68},
  year    = {2018}
}

@article{foster2023orthogonal,
  author  = {Foster, Dylan J. and Syrgkanis, Vasilis},
  title   = {Orthogonal Statistical Learning},
  journal = {The Annals of Statistics},
  volume  = {51},
  number  = {3},
  pages   = {879--908},
  year    = {2023},
  doi     = {10.1214/23-AOS2258}
}

@article{khan2010irregular,
  author  = {Khan, Shakeeb and Tamer, Elie},
  title   = {Irregular Identification, Support Conditions, and Inverse Weight Estimation},
  journal = {Econometrica},
  volume  = {78},
  number  = {6},
  pages   = {2021--2042},
  year    = {2010}
}

@inproceedings{bhattacharyya2022efficient,
  author    = {Bhattacharyya, Arnab and Gayen, Sutanu and Kandasamy, Saravanan and Raval, Vedant and Vinodchandran, N. V.},
  title     = {Efficient Interventional Distribution Learning in the {PAC} Framework},
  booktitle = {Proceedings of the 25th International Conference on Artificial Intelligence and Statistics},
  series    = {Proceedings of Machine Learning Research},
  volume    = {151},
  pages     = {7531--7549},
  publisher = {PMLR},
  year      = {2022}
}

@inproceedings{choo2025probably,
  author    = {Choo, Davin and Squires, Chandler and Bhattacharyya, Arnab and Sontag, David},
  title     = {Probably Approximately Correct High-Dimensional Causal Effect Estimation Given a Valid Adjustment Set},
  booktitle = {Proceedings of the Fourth Conference on Causal Learning and Reasoning},
  series    = {Proceedings of Machine Learning Research},
  volume    = {275},
  pages     = {1032--1085},
  publisher = {PMLR},
  year      = {2025}
}

@article{liu2026sample,
  author  = {Liu, Bo and Yang, Chengxin and Li, Fan},
  title   = {Sample Size and Power Calculations for Causal Inference in Observational Studies},
  journal = {The Annals of Statistics},
  volume  = {54},
  number  = {4},
  pages   = {2191--2208},
  year    = {2026}
}

@article{rosenbaum1983central,
  author  = {Rosenbaum, Paul R. and Rubin, Donald B.},
  title   = {The Central Role of the Propensity Score in Observational Studies for Causal Effects},
  journal = {Biometrika},
  volume  = {70},
  number  = {1},
  pages   = {41--55},
  year    = {1983},
  doi     = {10.1093/biomet/70.1.41}
}

@article{shooksa2022power,
  author  = {Shook-Sa, Bonnie E. and Hudgens, Michael G.},
  title   = {Power and Sample Size for Observational Studies of Point Exposure Effects},
  journal = {Biometrics},
  volume  = {78},
  number  = {1},
  pages   = {388--398},
  year    = {2022}
}

@article{valiant1984theory,
  author  = {Valiant, Leslie G.},
  title   = {A Theory of the Learnable},
  journal = {Communications of the ACM},
  volume  = {27},
  number  = {11},
  pages   = {1134--1142},
  year    = {1984},
  doi     = {10.1145/1968.1972}
}

@inproceedings{yin2025sampling,
  author    = {Yin, Yue and Yang, Jiaoyun and An, Ning and Li, Lian},
  title     = {Sampling Boundary for Causal Effect Estimation},
  booktitle = {Proceedings of the 17th Asian Conference on Machine Learning},
  series    = {Proceedings of Machine Learning Research},
  volume    = {304},
  pages     = {527--542},
  publisher = {PMLR},
  year      = {2025}
}

@inproceedings{tadepalli2021pac,
  author    = {Tadepalli, Prasad and Russell, Stuart J.},
  title     = {{PAC} Learning of Causal Trees with Latent Variables},
  booktitle = {Proceedings of the AAAI Conference on Artificial Intelligence},
  volume    = {35},
  pages     = {9774--9781},
  year      = {2021}
}

@article{lecam1973convergence,
  author  = {Le Cam, Lucien},
  title   = {Convergence of Estimates under Dimensionality Restrictions},
  journal = {The Annals of Statistics},
  volume  = {1},
  number  = {1},
  pages   = {38--53},
  year    = {1973}
}

@book{tsybakov2009introduction,
  author    = {Tsybakov, Alexandre B.},
  title     = {Introduction to Nonparametric Estimation},
  publisher = {Springer},
  address   = {New York},
  year      = {2009}
}

@article{donoho1991geometrizing,
  author  = {Donoho, David L. and Liu, Richard C.},
  title   = {Geometrizing Rates of Convergence, {II}},
  journal = {The Annals of Statistics},
  volume  = {19},
  number  = {2},
  pages   = {633--667},
  year    = {1991}
}

@article{birge1983approximation,
  author  = {Birg{\'e}, Lucien},
  title   = {Approximation dans les espaces m{\'e}triques et th{\'e}orie de l'estimation},
  journal = {Zeitschrift f{\"u}r Wahrscheinlichkeitstheorie und verwandte Gebiete},
  volume  = {65},
  number  = {2},
  pages   = {181--237},
  year    = {1983}
}

@article{yang1999information,
  author  = {Yang, Yuhong and Barron, Andrew},
  title   = {Information-Theoretic Determination of Minimax Rates of Convergence},
  journal = {The Annals of Statistics},
  volume  = {27},
  number  = {5},
  pages   = {1564--1599},
  year    = {1999}
}

@article{gleser1987nonexistence,
  author  = {Gleser, Leon Jay and Hwang, Jiunn T.},
  title   = {The Nonexistence of $100(1-\alpha)\%$ Confidence Sets of Finite Expected Diameter in Errors-in-Variables and Related Models},
  journal = {The Annals of Statistics},
  volume  = {15},
  number  = {4},
  pages   = {1351--1362},
  year    = {1987}
}

@article{dufour1997some,
  author  = {Dufour, Jean-Marie},
  title   = {Some Impossibility Theorems in Econometrics with Applications to Structural and Dynamic Models},
  journal = {Econometrica},
  volume  = {65},
  number  = {6},
  pages   = {1365--1387},
  year    = {1997}
}
